%% file: iclr2026_conference.tex
\documentclass{article} 
\usepackage{iclr2026_conference,times}

\input{math_commands.tex}

\definecolor{brandBlue}{HTML}{CDC1FF} 
\definecolor{mycolor}{HTML}{4CC9FE} 
\usepackage{url}
\usepackage{hyperref}
\usepackage{wrapfig}
\usepackage{tikz,lipsum,lmodern}
\usepackage[most]{tcolorbox}
\usepackage[dvipsnames]{xcolor}
\usepackage{xurl}
\usepackage{amsmath, amssymb, amsthm}
\usepackage{nicefrac} 
\theoremstyle{definition}
\newtheorem{theorem}{Theorem}
\newtheorem{definition}{Definition}
\newtheorem{lemma}{Lemma}
\newtheorem{corollary}{Corollary}
\usepackage{xcolor}

\usepackage[utf8]{inputenc} 
\usepackage[T1]{fontenc}    
\usepackage{hyperref}       
\usepackage{url}            
\usepackage{booktabs}       
\usepackage{amsfonts}       
\usepackage{nicefrac}       
\usepackage{microtype}      
\usepackage[table]{xcolor}  
\usepackage{graphicx}       
\usepackage{multirow}       
\usepackage{tabularx}       
\usepackage{amsmath}        
\usepackage{enumitem}       
\usepackage{ulem}

\usepackage{tikz}
\usetikzlibrary{arrows.meta,positioning}
\usepackage[svgnames,x11names]{xcolor}

\usepackage{arydshln}
\usepackage{booktabs}        
\usepackage{siunitx}         
\usepackage[table]{xcolor}   
\usepackage{bookmark}
\usepackage{subcaption} 
\newcommand{\shortname}{\texttt{APRM}}
\hypersetup{
    colorlinks=true,
    citecolor=DarkBlue,
    linkcolor=DarkBlue,
    urlcolor=DarkBlue
}

\title{Adversarial Training for Process Reward \\Models}

\author{Gurusha Juneja, Deepak Nathani, William Yang Wang \\
Department of Computer Science\\
University of California, Santa Barbara\\
\texttt{gurusha@ucsb.edu} \\
}

\usepackage{enumitem}

\iclrfinalcopy 
\begin{document}

\maketitle

\begin{abstract}
Process Reward Models (PRMs) enhance reasoning ability of LLMs by providing step-level supervision. 
However, their widespread adoption is limited due to expensive manual step-level annotation and poor generalization of static training data to novel errors. 
We introduce Adversarially Trained PRMs (\texttt{APRM}), where
a Generator ($G$) learns to produce reasoning errors to deceive a PRM ($R$), while $R$ concurrently learns to detect them.
This interaction yields progressively harder negatives for $R$, improving it's robustness and generalization to novel errors without requiring manual step-level labels.
Averaged across diverse mathematical reasoning benchmarks, \texttt{APRM} improves solver accuracy by $+3.4$ percentage points (pp) over the strongest PRM baseline. \texttt{APRM} achieves gains of $+5.3$ pp on out-of-distribution tasks.  \footnote{The code is available at: \href{https://gurusha01.github.io/PRM_NIPS/}{https://gurusha01.github.io/PRM\_NIPS/}} 
\end{abstract}
\textit{``It is not the strongest of the species that survives, nor the most intelligent, but the one most adaptable to change''} \hfill \textit{- Charles Darwin}
\section{Introduction}

\input{Introduction/problem}

\section{Adversarial Process Reward Models}
\input{Theory/0_theory_intro}
\input{Theory/2_reward_function}
\input{Theory/3_equillibrium_discussion}
\input{Theory/4_convergence_proof}

\input{Theory/5_Training_Algorithm}

\section{Experiments}
\input{Experiments/0_intro}
\input{Experiments/experimental_setup}
\input{Experiments/experimental_results}

\input{Experiments/qualitative_analysis}

\section{Related Work}

\input{Related_Work/reward_models}
\input{Related_Work/adversarial_ml}
\input{Discussion/discussion}
\newpage
\bibliography{iclr2026_conference}
\bibliographystyle{iclr2026_conference}
\newpage
\section{Appendix}
\input{Appendix/llm_use}
\input{Appendix/qualitaitve_examples}
\input{Appendix/oracle}

\input{Appendix/related_work}
\input{Appendix/prm_perf}

\end{document}

%% file: math_commands.tex
\usepackage{amsmath,amsfonts,bm}

\def\eqref#1{equation~\ref{#1}}

\def\1{\bm{1}}

\DeclareMathAlphabet{\mathsfit}{\encodingdefault}{\sfdefault}{m}{sl}
\SetMathAlphabet{\mathsfit}{bold}{\encodingdefault}{\sfdefault}{bx}{n}



%% file: Introduction/problem.tex
Large Language Models (LLMs) have demonstrated impressive capabilities in complex reasoning, even achieving gold-medal-level performance on the International Mathematical Olympiad questions, when augmented with specialized systems \citep{alphageometry2024, DeepMind2025GeminiIMO}. However, their reasoning remains fundamentally unreliable, often producing factually incorrect outputs despite appearing coherent \citep{chu2023investigating, bommasani2021opportunities, uesato2022solving}. This unreliability poses significant risks, particularly in high-stakes domains such as scientific discovery, finance, or medicine, where subtle errors can have severe consequences. As these models become better, their errors become more nuanced\citep{verifystepbystep, zhang2024rest, guan2025rstar, gpt4}.
Detecting such errors is essential for trustworthy reasoning, necessitating the need for Process Reward Models (PRMs) that can identifying very subtle errors.

Current PRM training techniques predominantly rely on static datasets \citep{lightman2023lets, wang2024math, prm800k}. These methods are inherently limited as they provide a fixed error distribution that cannot adapt to more nuanced errors. On the other hand, synthetic data generation techniques like \cite{zhang2024rest, lu2024autopsv, cobbe2021training, yu2024ovm, uesato2022solving} rely on the flawed assumption that correct final answers imply correct intermediate steps. These paradigms lack mechanisms to actively mine harder negatives. 

To address this, we require a PRM training paradigm that provides \emph{an adaptive curriculum, where negative sample hardness dynamically increases while aligning with the PRM's evolving capabilities}. This makes the PRM robust to increasingly subtle errors and also ensures optimal and smooth learning. In this paper, we introduce Adversarially Trained Process Reward Models (\shortname{}). We formulate PRM training as a two-player, general-sum, non-cooperative game. In \shortname{} (Fig \ref{fig:main-fig}), a Generator ($G$) actively learns to produce plausible but incorrect reasoning steps to deceive a Reward Model ($R$), which concurrently learns to detect these increasingly subtle errors. This competitive interaction creates the desired adaptive curriculum, directly addressing the limitations of static or passively generated training data. 
\begin{figure*}
    \centering
    \includegraphics[width=\linewidth]{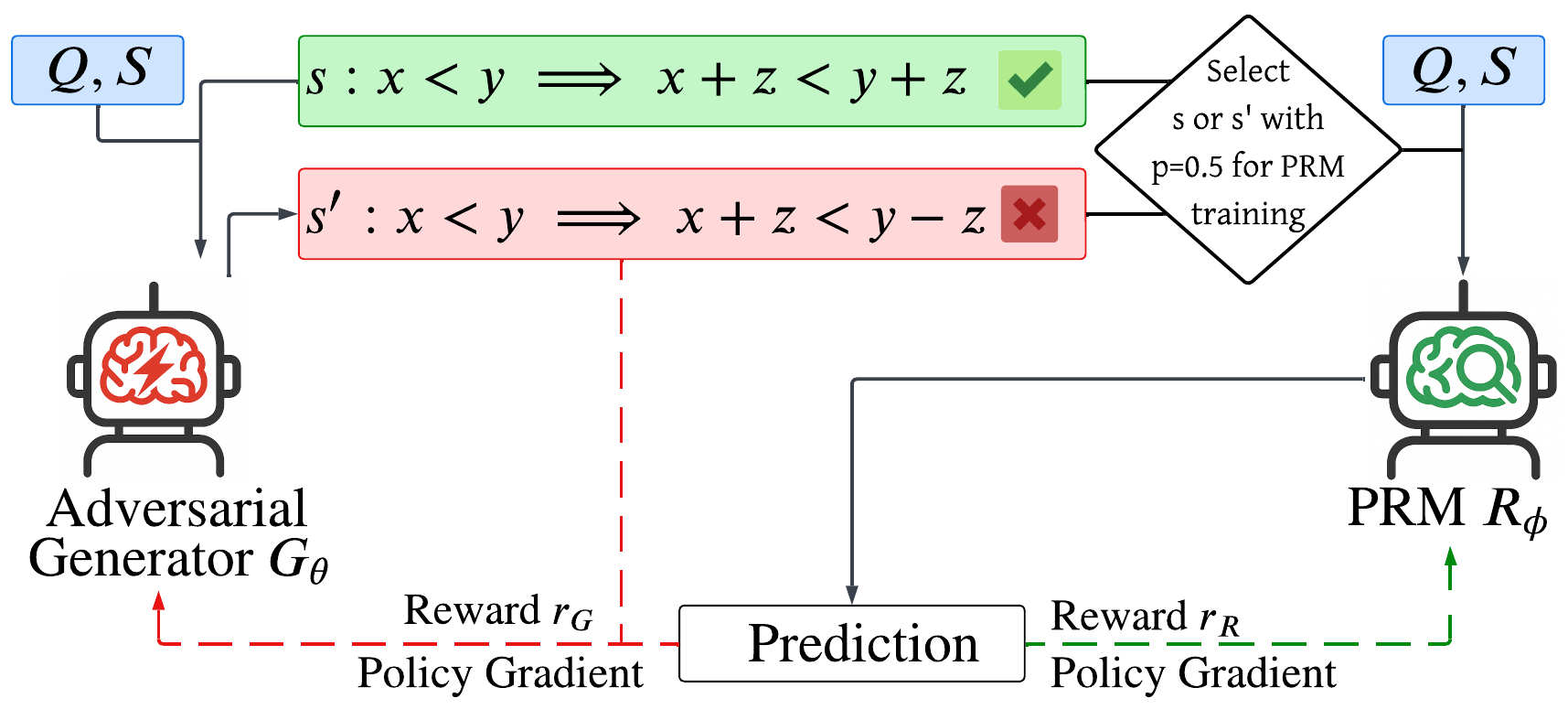}
    \caption{\textbf{\shortname{} Overview.} 
    The Generator $G_\theta$ perturbs the correct step $s$ into an incorrect step $s'$. The PRM $R_\phi$ evaluates a step for it's correctness given the question $Q$ and partial solution $S$. Both $G_\theta$ and $R_\phi$ optimize for rewards given based on the perturbed step and $R_\phi$'s prediction, details in section ~\ref{sec:reward_definitions}. }
    \vspace{-1.5em}
    \label{fig:main-fig}
\end{figure*}

\shortname{} is both theoretically and empirically grounded. We employ 
game-aware optimizers and symmetric policy regularization to ensure stable training dynamics with provable linear-rate convergence to a Nash Equilibrium. Unlike prior adversarial approaches~\citep{gan, madry2018towards, sppo} that were framed as zero-sum games, we model the Generator–PRM interaction as a general-sum game. This yields a stationary-point Nash Equilibrium rather than a mini-max saddle point. Experiments on mathematical reasoning benchmarks show that \shortname{} outperforms state-of-the-art PRM training and prompting methods by up to 5.3 pp, with the largest gains on out-of-distribution problems, highlighting it's enhanced robustness and generalization.

We summarize our contribution as follows:
\vspace{-1em}
\begin{itemize}[leftmargin=*, noitemsep]
\item	\textbf{Adversarial training for PRMs.} We introduce PRM training as a two-player adversarial game between a generator and a PRM. The generator learns to produce harder negatives, leading to stronger and more generalizable Process Reward Models.
\item	\textbf{Theoretical Analysis.} We provide theoretical analysis of the game, showing the existence of Nash Equilibria. We establish linear convergence guarantees under regularized game utilities and use game-aware optimizers to enable stable training.
\item	\textbf{Experimental validation.} We experimentally validate the proposed adversarial training setup. Across math and science reasoning benchmarks, \shortname{} yields an average improvement of 3.4 pp over state-of-the-art PRM training methods, consistently observed across solvers of different sizes.
\end{itemize}

\vspace{-1em}




%% file: Theory/0_theory_intro.tex
\label{sec:notations_full_reference}

We formulate the problem of learning a robust process reward model (PRM) as a two-player adversarial game. This game consists of a Generator $G$, tasked with perturbing a correct reasoning steps to make it incorrect, and a Reward Model $R$, whose objective is to predict the correctness of a given reasoning step. $G$ aims to deceive $R$ into classifying the corrupted step as correct. $R$, on the other hand, aims for accurate classification, detecting even the most nuanced errors generated by $G$ (Figure \ref{fig:main-fig}). Through this iterative competition, $G$ learns to generate increasingly subtle errors, while $R$ concurrently evolves to detect even the most nuanced flaws.

\paragraph{Notation.} 
We denote the Generator by $G_{\theta}$ and the Reward Model by $R_{\phi}$, parameterized by weights $\theta \in \Theta$ and $\phi \in \Phi$, respectively. A correct reasoning step is represented as a sequence $s = \{s_1, s_2, \dots, s_m\}$, where each $s_i$ is a token from the finite vocabulary $\mathcal{V}$.

The Generators action is a token sequence $a_G \in \mathcal{A}_G$, where $a_G = \{s'_1, s'_2, \dots, s'_n\}$ with $s'_i \in \mathcal{V}$ and $n \leq \ell$ ($\ell$ is the maximum sequence length). The action space $\mathcal{A}_G$ thus consists of all token sequences of length $1$ to $\ell$. These generated sequences correspond to corrupted reasoning steps.

The Reward Model’s action is a binary classification $a_R \in \mathcal{A}_R = \{0,1\}$, where $1$ denotes \emph{Correct} and $0$ denotes \emph{Incorrect}. The players use mixed strategies(policies) defined as $\pi_{\theta} \in \Delta(\mathcal{A}_G)$ for the Generator and $\pi_{\phi} \in \Delta(\mathcal{A}_R)$ for the Reward Model, where $\Delta(\cdot)$ denotes a probability distribution over the respective action spaces. Note that these strategies are distinct from the underlying model parameters $\theta$ and $\phi$.

Let $y(a_G) \in \{0, 1\}$ be the ground-truth correctness label of a reasoning step generated by $G_\theta$, i.e. $y=1$ for a correct generated step and, $0$ for incorrect generated step, $a_R=y'(a_G) \in \{0, 1\}$ be $R_\phi$'s classification output for $G$'s action. The scalar reward functions for $G_\theta$ and $R_\phi$ are represented by $r_G(y, y')$ and $r_R(y, y')$,
jointly determined by the ground-truth validity of $G_\theta$'s output and $R_\theta$'s classification. Detailed definitions are provided in Section~\ref{sec:reward_definitions}

Utilities represent the expected payoffs for each player. Given a strategy profile $(\pi_{\theta}, \pi_{\phi})$ and a distribution $p(s)$ over correct initial reasoning steps $s$, the utilities for both the players are given by:
\begin{equation}
\begin{split}
    U_G(\pi_{\theta}, \pi_{\phi}) = \mathbb{E}_{s \sim p(s), a_G \sim \pi_{\theta}(\cdot|s), y' \sim \pi_{\phi}(\cdot|a_G)} [r_G(y(a_G), y')] \\
    U_R(\pi_{\theta}, \pi_{\phi}) = \mathbb{E}_{s \sim p(s), a_G \sim \pi_{\theta}(\cdot|s), y' \sim \pi_{\phi}(\cdot|a_G)} [r_R(y(a_G), y')]
\end{split}
\end{equation}

%% file: Theory/2_reward_function.tex
\subsection{Reward Function}
\label{sec:reward_definitions}


We use reinforcement learning (RL) to train $G_\theta$ and $R_\phi$ through iterative adversarial interaction. The correctness $y(a_G)$ of $G$'s generation is determined by an algorithmic oracle that uses non-LLM metrics, such as cosine similarity, entity matching (nouns and numbers), and logical operation order, relative to a reference solution. Further details of this oracle are provided in Appendix~\ref{app:oracle}. Given $y(a_G)$ and the Reward Model's classification $y'(a_G)$, the reward functions for both agents are structured as follows:



\paragraph{Reward Function for $R_\phi$}
The PRM is rewarded for accurate classifications. It receives $+1$ for correct classifications ($y=y'$) and incurs a penalty of $-1$ for incorrect classifications ($y \neq y'$).
\begin{equation}
    r_R(y, y') = \begin{cases} +1 & \text{if } y=y' \\ -1 & \text{if } y\neq y' \end{cases}
\end{equation}

\paragraph{Reward Function for $G_\theta$}
The Generator receives rewards that incentivize generation of deceptive errors. It is rewarded $+1$ for generating an incorrect reasoning step that successfully deceives $R$ into classifying it as correct ($y=0 \text{ and } y'=1$). It receives a reward of $-1$ for failing to corrupt a step (generating a correct step, $y=1$) and $0$ for failing to decieve $R_\phi$ when it generated an incorrect step ($y=0 \text{ and } y'=0$).
\begin{equation}
    r_G(y, y') = \begin{cases} +1 & \text{if } y=0 \text{ and } y'=1 \text{ (G generates incorrect, fools } R \text{)} \\ 0 & \text{if } y=0 \text{ and } y'=0 \text{ (G generates incorrect, but is caught by } R \text{)} \\ -1 & \text{if } y=1 \text{ (G produces correct step)} \end{cases}
\end{equation}

Note that since the reward/payoff functions of the two players are asymmetric (don't add up to the same value in every case), the game is formulated as a general-sum game. 

%% file: Theory/3_equillibrium_discussion.tex
\subsection{Game Formulation and Equilibrium Concept}
\label{sec:game_formulation_eq}
\shortname{} formulates PRM training as a two-player, non-cooperative game, general-sum. The natural solution concept for this adversarial interaction is a Nash Equilibrium~\citep{Nash1951}.

\paragraph{Utility Functions.} To ensure good optimization behavior (discussed in Section~\ref{sec:convergence}), our players optimize regularized utility functions. Both $G_\theta$'s and $R_\phi$'s utility is augmented with a Kullback-Leibler (KL) divergence penalty and an entropy bonus that help stabilize learning. 
These regularized utilities that define the precise objectives of the game, are defined as:
\begin{equation}
\begin{split}
U'_G(\pi_{\theta}, \pi_{\phi}) &= U_G(\pi_{\theta}, \pi_{\phi})  - \tau \cdot \text{KL}(\pi_{\theta} || \pi_{\theta_0}) + c_H \cdot H(\pi_{\theta}) \\
U'_R(\pi_{\theta}, \pi_{\phi}) &= U_R(\pi_{\theta}, \pi_{\phi})  - \tau \cdot \text{KL}(\pi_{\phi} || \pi_{\phi_0})+ c_H \cdot H(\pi_{\phi})
\end{split}
\label{eq:utility_reg}
\end{equation}
Where $U_G(\pi_{\theta}, \pi_{\phi}) \text{ and } U_R(\pi_{\theta}, \pi_{\phi})$ are player utilities defined as expected payoffs(rewards) in Section~\ref{sec:notations_full_reference}.  $\tau, c_H  > 0$ are coefficients for KL-regularization and entropy bonuses, respectively. $\pi_{\theta_0}$ and $\pi_{\phi_0}$  are the policies of the initial pretrained models. See Appendix~\ref{app:mode-collapse} for a discussion on mode collapse avoidance using these regularization.

\paragraph{Theorem 1 (Existence of Nash Equilibrium)}
The adversarial game, characterized by agents $G_{\theta}$ and $R_{\phi}$, their regularized utility functions ($U'_G$ and $U'_R$), and action spaces ($\mathcal{A}_G$ and $\mathcal{A}_R$), is guaranteed to possess at least one Nash Equilibrium in mixed strategies.

\textbf{Proof}: We invoke \textsl{Glicksberg's generalization of Nash's existence theorem}~\citep{glicksberg1952further}. The theorem guarantees the existence of a Nash Equilibrium in a game if two conditions
are met:
\begin{enumerate}[leftmargin=*]
    \item \textsl{The players' strategy spaces are non-empty, convex, and compact}: Since the action spaces $\mathcal{A}_G, \mathcal{A}_R$  are finite and discrete, the corresponding policy spaces $\Delta(\mathcal{A}_G), \Delta(\mathcal{A}_R)$ are non-empty, convex, and compact (finite-dimensional simplices). \footnote{The convexity referred to in Glicksberg’s theorem is policy-space convexity (i.e. convexity of distributions over actions). This is different from convexity in the parameter space of a neural model (weights $\theta,\phi$), which is typically non-convex. Existence of an equilibrium in policy space does not imply that there is a single weight vector achieving that equilibrium, or that gradient-based training will find it. Nor do we claim in this paper that we are able to find the Nash equilibrium.}
    \item \textsl{The players' utility functions are continuous in all players' strategies}: The utility functions $U'_G(\pi_\theta,\pi_\phi)$ and $U'_R(\pi_\theta,\pi_\phi)$ are expectations of bounded rewards under mixed (policy) distributions over actions. Because these expected‐reward expressions are finite linear combinations (or integrals) of policy probabilities, they(including regularization) are continuous in the policy distributions $\pi_\theta$ and $\pi_\phi$.
\end{enumerate}
Given these conditions, Glicksberg's theorem ensures the existence of at least one Nash Equilibrium.

\paragraph{Nash Equilibrium as a Stationary Point.}
In a general-sum, non-cooperative games where the utilities are differentiable w.r.t policy parameter, a Nash Equilibrium is characterized as a \textsl{stationary point} in the joint strategy space~\citep{stationary1, stationary2}. 
At such a point, no agent has a local incentive to unilaterally change their strategy. For \shortname{}, this equilibrium signifies a state where $G_{\theta}$'s evolving ability to corrupt is optimally countered by $R_{\phi}$'s ability to detect those corruptions (given their capacities and the training data).

A \textsl{stationary point} is qualitatively different from the \textsl{saddle point} solution that frequently arises in works like Generative Adversarial Networks \citep{gan} and robust optimization \citep{madry2018towards}. These problems are formulated as zero-sum games where the equilibrium corresponds to a saddle point of a single objective function optimized by both players in perfectly opposing directions, contrary to our general sum formulation.


%% file: Theory/4_convergence_proof.tex
\subsection{Convergence to Nash Equilibrium}
\label{sec:convergence}
In this section, we justify the introduction of regularization into the utility functions. The game described by the non-regularized utility functions $U_G\text{ and }U_R$ has non-monotone dynamics because the utilities are not simply negative correlated. In non-monotone settings, standard gradient methods often exhibit undesirable behaviors such as cycling, chaotic dynamics, or divergence, thus precluding direct convergence guarantees \citep{gidel2018variational, fei2021provably}. To mitigate these stability challenges and ensure provable convergence, we incorporate symmetric policy regularization into the player utility functions, as defined in Equation~\ref{eq:utility_reg}. Past works \citep{azar2024general, munos2024nash} prove that this regularization induces strong concavity in individual player objectives, which leads to a strong monotonicity and hence linear convergence to Nash Equilibrium.

\textbf{Lemma 1 (Strong Concavity from Regularization):} For any utility function $U(\pi)$ and a fixed reference policy $\pi_0$, the regularized objective $U'(\pi) = U(\pi) - \tau \cdot \text{KL}(\pi || \pi_0)$ is $\tau$-strongly concave with respect to the policy $\pi$. Similarly, an entropy bonus $\tau \cdot H(\pi)$ induces $\tau$-strong concavity~\citep{azar2024general}.

Applying the lemma
, and redefining the utility functions as given by Equation~\ref{eq:utility_reg}, 
we ensure that, $U'_G(\pi_{\theta}, \pi_{\phi})$ and $U'_R(\pi_{\theta}, \pi_{\phi})$ are $\tau$ and $c_H$-strongly concave with respect to the player policies.
If each player's utility function is strongly concave with respect to its own policy, for any fixed opponent policy, then the operator $F(z) = (-\nabla(U_\theta(\pi_\theta, \pi_\phi)), -\nabla(U_\phi(\pi_\theta, \pi_\phi)))$, where $z =(\theta, \phi)$, becomes strongly monotone \citep{daskalakis2018gans}. Furthermore, for games characterized by a strongly monotone operator, game-aware first-order optimizers offer strong convergence guarantees, specifically ensuring linear convergence rates.

\textbf{Theorem 2 (Linear Convergence to Nash Equilibrium):} For the strongly monotone game described by the regularized utility functions (Section~\ref{sec:game_formulation_eq}) algorithms such as Optimistic Gradient Descent-Ascent (OGDA) or Extra-Gradient \citep{ogda1, ogda2} are guaranteed to generate a sequence of iterates $z_t$ that converges to the unique Nash Equilibrium $z^*$ at a linear rate. (For proof, see Appendix~\ref{app:thm2}).

%% file: Theory/5_Training_Algorithm.tex
\subsection{Training Algorithm}
\label{sec:training_algorithm}

Both $G_\theta$ and $R_\phi$ are trained using PPO~\citep{ppo} augmented with Optimistic Gradient Descent-Ascent (OGDA) updates (see Section~\ref{sec:convergence}). 

\paragraph{Loss Function.}
For each player, we augment the PPO clipped surrogate objective with KL and Entropy regularization terms. For instance, $G_\theta$'s the loss function (similarly for $R_\phi$) is given by:
\begin{equation}
    L_G^\text{total}(\theta) = L_G^{\text{PPO,clip}}(\theta) + L_G^{\text{KL}}(\theta) + L_G^H(\theta)\\
\end{equation}
Where minimizing $L^{PPO}$ approximately maximizes the un-regularized utility functions $U_G, U_R$ defined in Section \ref{sec:notations_full_reference}. $L^{\text{KL}}$ given by $\tau \cdot \mathbb{E}_t \left[ \text{KL}(\pi_{\theta} || \pi_{\theta_0})) \right]$ and $L^H$ given by $-c_H \cdot \mathbb{E}_t \left[ \text{H}(\pi_{\theta})\right]$ correspond to the regularization introduced in Section \ref{sec:convergence}.

\paragraph{Game-Aware Optimization. }
We replace the standard PPO's gradient updates by OGDA  \citep{ogda1, ogda2}, which efficiently navigates the rotational game dynamics, in the regularized strongly monotone setting, by incorporating past gradient information. The update rule for OGDA is given by:
\begin{equation}
    z_{t+1} = z_t - \eta (2F(z_t) - F(z_{t-1}))
\end{equation}
where, $\eta$ is the learning rate, and $F(z_t)$ is is the gradient function w.r.t. the player's policy parameters, as described in section~\ref{sec:convergence}.  


%% file: Experiments/0_intro.tex
\label{sec:experiments}


\begin{table*}[!t]
  \centering
  \caption{Performance comparison of \shortname{} with the best MCTS based PRM training and prompting methods on math reasoning benchmarks. Please see Table \ref{tab:math-full} for a full comparison with other PRM training baselines.}
  \label{tab:math_performance}
\resizebox{\linewidth}{!}{
  \begin{tabular*}{\linewidth}{@{\extracolsep{\fill}}l*{6}{S}}
    \toprule
    \textbf{Method}
    & \textbf{MATH500}
    & \textbf{JEEB.}
    & \textbf{OlympiadB.}
    & \textbf{AIME25}
    & \textbf{AMC} 
    & \textbf{Avg.}\\
    \midrule
     \textbf{\texttt{GPT-OSS-120B}}\\
   \hphantom{xxxx}CoT-SC              & 84.4 & 62.8 & 83.9 & 85.5 & 60.0 & 77.2 \\
    \hphantom{xxxx}LLM-J      & 85.2 & 66.0 & 83.1 & 90.4 & 80.0 & 79.3\\
    \hphantom{xxxx}ToT          & 87.6 & 66.4 & 84.3 & 89.2 & 66.7 & 79.9 \\
    \hphantom{xxxx}ReST-MCTS        & \textbf{91.6} & 69.1 & 84.3 & 91.0 & 67.4 & 81.7 \\ 
    \hphantom{xxxx}Hard Neg. & 91.2 & 69.5 & 79.2 &  91.5 & 68.5 & 80.0\\
    \hphantom{xxxx}ARM & 89.2 & 62.4 & 77.6 & 87.9 & 63.3 & 76.1\\
    \rowcolor{mycolor!20}
    \hphantom{xxxx}\textbf{\shortname{}} (Ours) & 91.4 & \textbf{70.3} & \textbf{89.4} & \textbf{94.5} & \textbf{70.7} & \textbf{83.0}\\
    \midrule
    \textbf{\texttt{GPT-OSS-20B}}\\
    \hphantom{xxxx}CoT-SC              & 78.0 & 48.4 & 61.0 & 69.9 & 33.3 & 62.2 \\
    \hphantom{xxxx}LLM-J      & 85.2 & 68.0 & 83.9 & 92.8 & 60.0 & 80.3 \\
    \hphantom{xxxx}ToT          & 87.6 & 67.6 & 85.6 & 90.4 & 66.7 & 80.7 \\
    \hphantom{xxxx}ReST-MCTS        & \textbf{91.0} & 68.3 & 84.7 & 91.3 & 66.0 & 82.3 \\ 
    \rowcolor{mycolor!20}
    \hphantom{xxxx}\textbf{\shortname{}} (Ours) & \textbf{91.0} & \textbf{73.0} &\textbf{ 90.4} & \textbf{93.6} & \textbf{68.0} & \textbf{85.0} \\
    \midrule
    \textbf{\texttt{Gemma-3-27B}}\\
    \hphantom{xxxx}CoT-SC              & 78.8 & 48.8 & 48.7 & 69.9 & 13.3 & 58.4 \\
    \hphantom{xxxx}LLM-J      & 86.4 & 69.2 & 83.1 & 91.6 & 73.3 & 80.4 \\
    \hphantom{xxxx}ToT          & 88.0 & 67.6 & 85.2 & 88.0 & 60.0 & 80.2 \\
    \hphantom{xxxx}ReST-MCTS        & \textbf{93.2} & 70.3 & 84.3 & 90.6 & \textbf{61.2} & 72.6 \\ 
    \hphantom{xxxx}Hard Neg. & 86.0 & 63.0 & 83.0 & 89.1 & 60.0 & 76.2\\
    \hphantom{xxxx}ARM& 86.0 & 51.0 & 61.0 & 71.4 & 30.0& 59.9\\
    \rowcolor{mycolor!20}
    \hphantom{xxxx}\textbf{\shortname{}} (Ours)& 91.4 & \textbf{74.4} & \textbf{90.7} & \textbf{91.9} & 60.8 & \textbf{85.2} \\
    \midrule
    \textbf{\texttt{Gemma-3-12B}}\\
    \hphantom{xxxx}CoT-SC              & 76.8 & 43.2 & 35.6 & 53.0 & 23.3 & 51.4 \\
    \hphantom{xxxx}LLM-J      & 72.8 & 40.4 & 37.7 & 55.4 & 20.0 & 49.9 \\
    \hphantom{xxxx}ToT          & 77.6 & 49.2 & 40.7 & 61.4 & 20.0 & 55.3 \\
    \hphantom{xxxx}ReST-MCTS        & \textbf{81.4} & 50.7 & 40.3 & 62.6 & 20.2 & 56.9 \\ 
    \rowcolor{mycolor!20}
    \hphantom{xxxx}\textbf{\shortname{}} (Ours) & 80.0 & \textbf{53.1} & \textbf{44.0} &\textbf{ 65.7 }& \textbf{21.6} & \textbf{58.6} \\
    \bottomrule
  \end{tabular*}} 
  \vspace{-1em}
\end{table*}

%% file: Experiments/experimental_setup.tex
\subsection{Experimental Setting}

\textbf{Model Architecture  } Both the Generator ($G_\theta$) and the Process Reward Model ($R_\phi$) use \texttt{Llama-3.1-8B} as the backbone. This model offers a practical trade-off between computational cost and capacity for mathematical reasoning.


\textbf{Training  } We train on the train split of MATH dataset\citep{MATH}. The dataset provides with gold solution steps, the generator $G_\theta$ learns to perturb these solution steps. We alternate training between players, updating one for 5 gradient update steps while freezing the other. We use PPO \citep{ppo} with Optimistic Gradient Descent–Ascent (OGDA) \citep{ogda1} as discussed in Section~\ref{sec:training_algorithm}. $R_\phi$ is trained on a mixture of gold steps (50\%) and negatives drawn from both current and past versions of $G_\theta$ (50\%) to prevent forgetting. Hyperparameter details are given in the App~\ref{app:hyperparam}.


\textbf{Baselines } We compare against methods with comparable test time compute, which includes a) prompting-based methods b) trained reward models. Prompting based methods include CoT with self-consistency (k=5) \citep{self-consistency}, Tree of Thoughts (branching factor 5) \citep{treeofthought} and LLM-as-a-Judge~\citep{llmaj} using same model as solver and as judge sampling a maximum of 5 times. Trained reward models include outcome-based approximations methods - AutoPSV \citep{AutoPSV}, trained PRM using static datasets - Let's Verify Step by Step \citep{verifystepbystep}, and MCTS based data generation techniques to train PRMs - Math-Shepherd \citep{Math-Shepherd}, rStar-Math \citep{guan2025rstar} and ReST-MCTS \citep{zhang2024rest}. We re-implement all the methods using \texttt{Llama-3.1-8B} as a base model. Table \ref{tab:math_performance} gives performance comparison with prompting and the best trained PRM baseline. A full comparison with all the baselines is given in Table \ref{tab:math-full}.

\textbf{Test Setup  } We evaluate on five math reasoning benchmarks: MATH500~\citep{MATH} (in-domain, in-distribution), JEEBench~\citep{jeebench}, OlympiadBench~\citep{OlympiadBench}, AIME25~\citep{aime}, and AMC~\citep{aimo_aimo-validation-amc_2024} (in-domain, out-of-distribution). AIME25 contains recent examination questions, ensuring no data contamination. The effectiveness of trained PRMs is assessed by evaluating the performance of a solver model supervised by the PRM during test time. We sample a maximum of 5 times per step. We experiment on GPT and Gemma model families with different sizes (GPT-OSS-120B, GPT-OSS-20B, Gemma-3-12B-it and Gemma-3-27B-it) using API based inference. See App ~\ref{app:pareto-frontier} for compute details.

%% file: Experiments/experimental_results.tex
\subsection{Experimental Results}


\begin{table*}[!t]
  \centering
  \setlength{\tabcolsep}{3pt}
  \caption{Performance comparison \shortname{} with baselines on SciBench. 
  Domains: atk=Atkins, cal=Calculus, cmc=ChemMC, cls=Classical, dif=Differential Eq., 
  fun=Fundamentals, mat=Matter, qua=Quantum, sta=Statistics, the=Thermodynamics.}
  \label{tab:scibench_performance}
\resizebox{\linewidth}{!}{
  \begin{tabular*}{\linewidth}{@{\extracolsep{\fill}}l*{11}{c}}
    \toprule
    \textbf{Method} & \textbf{atk} & \textbf{cal} & \textbf{cmc} & \textbf{cls} & \textbf{dif} & 
    \textbf{fun} & \textbf{mat} & \textbf{qua} & \textbf{sta} & \textbf{the} & \textbf{Overall} \\
    
    \midrule
    
    \textbf{\texttt{GPT-OSS-120B}}\\
     \hspace{0.5em} ToT                & 58.3 & \textbf{96.2} & 43.5 & 51.6 & 74.1 & \textbf{37.5} & 46.4 & \textbf{50.0} & 89.1 & 61.0 & 61.4 \\
     \hspace{0.5em}  ReST-MCTS          & 58.3 & 88.5 & 43.5 & 52.1 & 74.1 & 35.0 & 50.0 & \textbf{50.0} & 91.3 & 70.7 & 62.3 \\
    \rowcolor{mycolor!20}
     \hspace{0.5em} \textbf{\shortname{}} (Ours) & \textbf{60.3} & 91.2 & \textbf{44.8 }& \textbf{52.4} & \textbf{76.9} & 36.6 & \textbf{51.5} & \textbf{50.0} & \textbf{92.5} & \textbf{73.1} & \textbf{64.0} \\
    \hline
    
    \textbf{\texttt{GPT-OSS-20B}}\\
        \hspace{0.5em} ToT      & 60.0 & \textbf{88.5} & 43.5 & 45.2 & 74.1 & 37.5 & 46.4 & 60.0 & 89.1 & 65.9 & 61.7 \\
    \hspace{0.5em} ReST-MCTS         & 61.3 & 86.2 & 43.5 & 44.5 & 71.5 & 37.5 & 50.0 & 60.0 & 89.1 & \textbf{66.2} & 61.8 \\
    \rowcolor{mycolor!20}
     \hspace{0.5em} \textbf{\shortname{}} (Ours) & \textbf{61.8} & 86.1 & \textbf{44.8} & \textbf{47.6} & \textbf{75.1} & \textbf{38.8} & \textbf{52.0} & \textbf{61.3} & \textbf{91.0} & 65.3 & \textbf{63.0} \\
    \hline

    \textbf{\texttt{Gemma-3-27B-IT}} \\
   \hspace{0.5em} ToT      & 60.0 & 92.3 & 43.5 & 48.4 & 77.8 & 30.0 & 50.0 & 50.0 & 89.1 & \textbf{63.4} & 61.1\\
   \hspace{0.5em} ReST-MCTS         & 60.0 & 92.2 & 43.5 & 48.4 & 77.8 & 31.5 & 57.1 & 50.0 & 89.1 & 63.2 & 61.8 \\
    \rowcolor{mycolor!20}
    \hspace{0.5em} \textbf{\shortname{}} (Ours) &\textbf{ 62.2} & \textbf{94.6} & \textbf{45.1} & \textbf{50.2} &\textbf{ 80.2} & \textbf{31.9} & \textbf{59.3} & \textbf{52.0} & \textbf{91.4} & \textbf{63.4} & \textbf{63.5} \\
    \midrule
    \textbf{\texttt{Gemma-3-12B-IT}}\\
   \hspace{0.5em} ToT          & 43.3 & \textbf{73.1} & 34.8 & \textbf{51.6} & \textbf{59.3} & \textbf{40.0} & 21.4 & \textbf{20.0} & \textbf{80.4} & \textbf{36.6} & \textbf{47.6} \\
    \hspace{0.5em} ReST-MCTS      & 43.3 & 69.2 & 34.8 & 35.5 & 44.4 & 32.5 & 25.0 & 10.0 & 76.1 & 26.8 & 41.8 \\
    \rowcolor{mycolor!20}
    \hspace{0.5em} \textbf{\shortname{}} (Ours)& \textbf{44.8} & 71.5 & \textbf{35.9} & 36.8 & 45.8 & 33.6 & \textbf{25.9} & 10.4 & 78.4 & 27.8 & 43.2 \\
    \bottomrule
  \end{tabular*}}
  \vspace{-0.5em}
\end{table*}

\paragraph{\shortname{} outperforms baselines}
Table~\ref{tab:math_performance}, shows the performance LLM solvers guided by our PRM against prior methods under equivalent test-time compute. Averaged over five benchmarks and four solver models, \shortname{} improves solver accuracy by $+3.4$ pp relative to the strongest trained PRM baseline (ReST-MCTS \citep{zhang2024rest}) and by $+4.2$ pp relative to the best prompting method. We observe the largest gains on JeeBench, which is an out-of-distribution benchmark, where \shortname{} outperforms ReST-MCTS by $+5.3$ pp, suggesting better generalization beyond the training distribution.
 
\vspace{-1em}
\paragraph{\shortname{} remains effective as with solver scales}
We evaluate whether \shortname{} remains effective as the solver size increases. Table~\ref{tab:math_performance} shows that \shortname{}’s average improvement over prompt-based methods grows from $+3.3$ pp with Gemma-3-12B, to $+4.3$ pp with GPT-OSS-20B, and further to $+5.0$ pp with Gemma-3-27B. In contrast, PRMs trained on non-adaptive data (e.g. ReST-MCTS) show gains that are smaller and in some cases decline as solver size increases, dropping from $+1.6$ pp on Gemma-3-12B to $-7.6$ pp on Gemma-3-27B. These results support our hypothesis that adversarial training enables robust PRMs that remain effective even as solver models become stronger.
\vspace{-1em}

\paragraph{\shortname{} generalizes across domains}
To test whether our adversarial training captures general error patterns rather than overfitting to math, we evaluate on SciBench~\citep{scibench}, comparing \shortname{} with the best prompting and PRM baselines (Tab.~\ref{tab:scibench_performance}). On GPT-OSS-20B, \shortname{} reaches $63.0\%$ accuracy versus $61.8\%$ for ReST-MCTS and $61.1\%$ for ToT. On GPT-OSS-120B, it achieves $64.0\%$, compared to $62.3\%$ and $61.4\%$. This shows that \shortname{} learns transferable properties of incorrect reasoning.

\paragraph{\shortname{} provides strong reward signal for RL Post-Training}
To assess \shortname{}'s utility beyond inference-time search, we evaluate its effectiveness as a reward signal for RL post-training. We fine-tune \texttt{Gemma-3-12B} on the MATH dataset using GRPO, comparing \shortname{} against standard outcome supervision and the strongest PRM baseline (ReST-MCTS). As shown in Table~\ref{tab:rl_results}, \shortname{} provides a significantly stronger learning signal, outperforming outcome-based RL by \textbf{+6.8 pp} on MATH500 and demonstrating superior generalization on OOD benchmarks like JEEBench (\textbf{+10.4 pp}). Training details are provided in Appendix~\ref{app:hyperparam}.
\begin{table}[h]
    \centering
    \caption{\textbf{RL Post-Training Performance.} Accuracy (\%) of \texttt{Gemma-3-12B} fine-tuned using different reward signals. APRM serves as a more effective dense reward than baselines.}
    \label{tab:rl_results}
    \resizebox{0.9\textwidth}{!}{
    \begin{tabular}{lccccc}
        \toprule
        \textbf{Method} & \textbf{MATH500} & \textbf{JEEB.} & \textbf{OlympiadB.} & \textbf{AIME25} & \textbf{AMC} \\
        \midrule
        Base Model (CoT) & 76.8 & 43.2 & 35.6 & 53.0 & 23.3 \\
        Final Correct (Outcome) & 80.0 & 48.2 & 40.0 & 59.8 & 21.0 \\
        ReST-MCTS PRM & 83.0 & 51.6 & 43.2 & 61.4 & 23.3 \\
        \textbf{APRM (Ours)} & \textbf{86.8} & \textbf{58.6} & \textbf{48.0} & \textbf{65.7} & \textbf{25.3} \\
        \bottomrule
    \end{tabular}
    }
\end{table}



\paragraph{Performance on PRM benchmarks}
We evaluate whether the predictions of \shortname{} align with human annotations by comparing it’s performance with the baselines on a subset of the PRM800K test set \citep{prm800k}. As shown in Figure~\ref{fig:rewards_and_metrics}, we report accuracy, precision, recall, and F1 score for all methods. We find that \shortname{} achieves the highest performance across all metrics, indicating stronger agreement with human-anotations compared to baselines.

\setlength{\intextsep}{5pt}   
\setlength{\columnsep}{10pt}  
\begin{wrapfigure}{r}{0.5\columnwidth}  
  \vspace{-5pt}
  \centering
  \includegraphics[width=\linewidth]{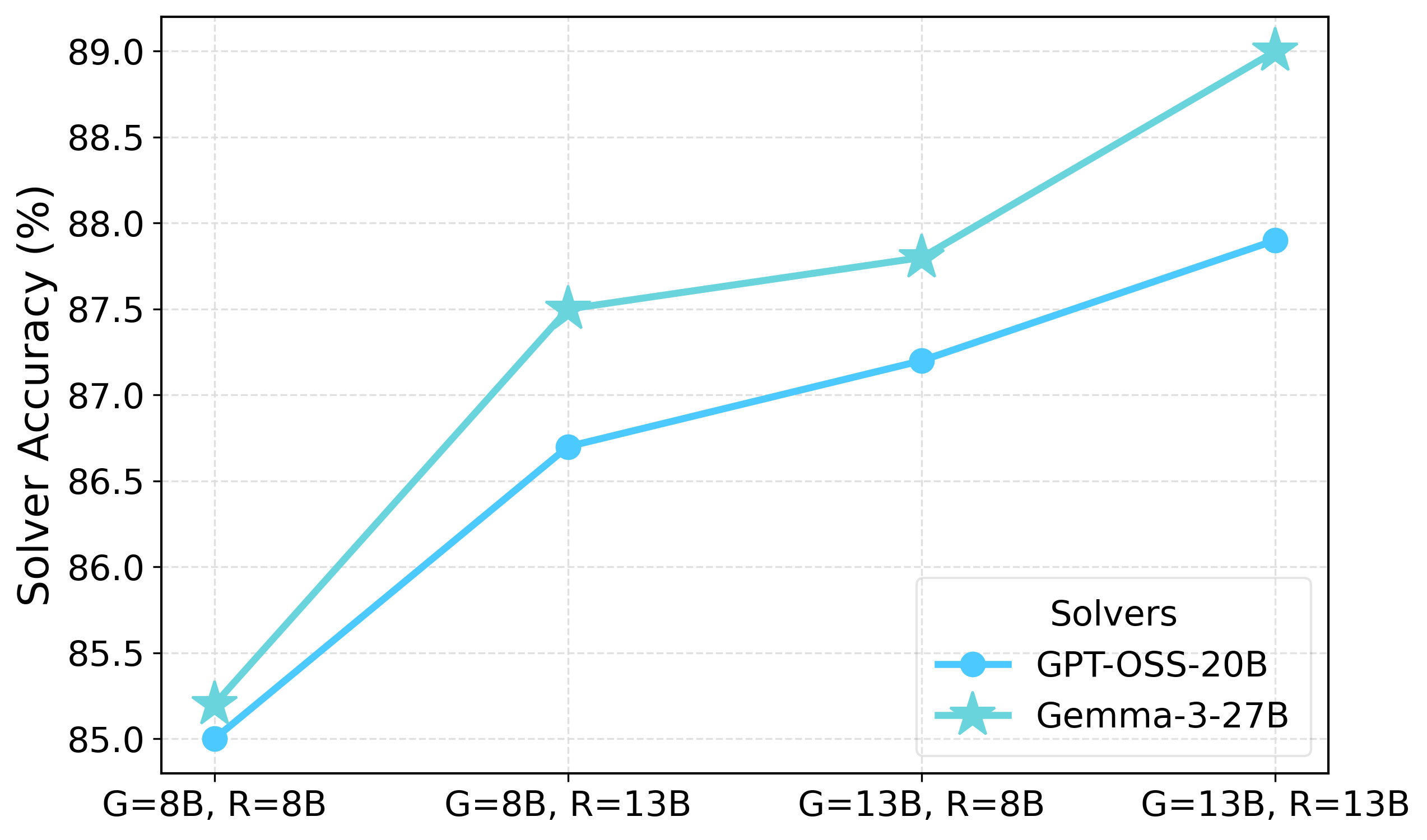}
  \vspace{-2em}
  \caption{\shortname{} augmented solver accuracy averaged across math reasoning benchmarks for different $R_\phi$ and $G_\theta$ backbones.}
  \label{fig:scaling}
  \vspace{-10pt}
\end{wrapfigure}

\paragraph{Scaling properties of \shortname{}} We study how \shortname{} scales with the backbone sizes of $G_\theta$ and $R_\phi$, using Llama-3.1-8B and Llama-2-13B in four settings: (1) $G$=8B, $R$=8B; (2) $G$=8B, $R$=13B; (3) $G$=13B, $R$=8B; and (4) $G$=13B, $R$=13B. Figure~\ref{fig:scaling} reports solver accuracy averaged over all math benchmarks, showing monotonic gains from (1) to (4) with an overall improvement of $\sim 3.8$ pp. We observe that gains are more when increasing the generator's capacity compared to increasing the reward model. We hypothesize that larger generators exposes the PRM to broader distribution reasoning errors.




\begin{wraptable}{r}{0.6\linewidth}
\centering
\setlength{\tabcolsep}{3pt}
\caption{Ablation study on \texttt{GPT-OSS-20B}. Accuracy (\%) on representative benchmarks.}
\label{tab:ablation}
\begin{tabular}{lcccc}
\toprule
\textbf{Variant} & \textbf{MATH} & \textbf{JEE} & \textbf{Oly} & \textbf{Overall} \\
\midrule
\shortname{} (full)          & \textbf{91.0} & \textbf{73.0} & \textbf{90.0} & \textbf{85.2} \\
    \hspace{1em}– no entropy reg.& 86.4 & 69.4 & 85.5 & 80.9 \\
    \hspace{1em}– no OGDA        & 84.6 & 67.9 & 83.7 & 79.2 \\
    \hspace{1em}– both removed   & 83.6 & 67.2 & 82.8 & 78.4 \\
\bottomrule
\end{tabular}
\end{wraptable}
\paragraph{Ablation Study}
We ablate on entropy regularization and the OGDA optimizer. Table~\ref{tab:ablation} shows results on MATH, JEE, and OlympiadBench using \texttt{GPT-OSS-20B}. Removing entropy regularization reduces accuracy by about 4 points, while removing OGDA has a similar effect. This complies with the theory, that each component helps stabilize training.

%% file: Experiments/qualitative_analysis.tex
\subsection{Qualitative Analysis}
\label{sec:qualitative_analysis}

\paragraph{Evolution of Generator perturbations.} 
We first analyze how the Generator’s perturbations evolve during training. In the die-rolling problem (Fig.~\ref{fig:qualitative-a}), $G$ initially perturbs the step to make a shallow mistake of as miscounting the number of odd faces. By mid-training, it produces steps that look locally consistent but misinterpret the problem statement, e.g., doubling all faces without considering conditional rules. At later stages, it introduces subtle logical errors, like averaging probabilities instead of taking a weighted average. This trajectory suggests that adversarial interaction trains $G$ to generate progressively harder negatives for $R$.

\paragraph{\shortname{} with different solver sizes.} 
We next examine \shortname{}-trained PRMs when the solver capacity increases. In the shortest-path problem (Fig.~\ref{fig:qualitaitve-b}), GPT-oss-20B uses a non-coherent languages and produces a confusing and incorrect coordinate setup, while GPT-oss-120B generates a more coherent geometric reasoning but still mislocates the position of gecko. In both cases, \shortname{} flags the faulty reasoning step, showing that adversarial training makes the PRM robust to errors even when the text is fluent.

\paragraph{Cross-domain transfer.} 
We test whether \shortname{} captures generalizable properties of flawed reasoning by applying it to a chemistry problem from SciBench (Fig.~\ref{fig:qualitative-c}). In this case, the solver neglected to convert centimeters to meters before applying $\Pi = \rho g h$. \shortname{}-trained PRM flagged this unit error, which lead to correct final answer. This example suggests that adversarially generated negatives are not tied to math-specific syntax errors but can transfer to scientific reasoning tasks.

\paragraph{Comparison with PRM baselines.} 
We further compare \shortname{} with ReST-MCTS on a calculus reasoning problem (Fig.~\ref{fig:qualitative-d}). The problem asks to calculate the value of an expression at $x=1$, where the numerator and denominator approach to zero. The solver factorizes and cancels the common factor and computes the value. This step appears valid, and even produces the correct final answer. However, this cancellation is invalid at $x=1$ point, because it amounts to dividing by zero. ReST-MCTS fails to flag this subtle mistake, whereas, \shortname{} flags this. This example highlights how adversarial training exposes PRMs to errors that lead to correct final answer, which static rollouts are not able to.

\paragraph{Failure Cases.} 
Finally, we analyze the errors missed by \shortname{}. We find that one recurring failure mode occurs when solvers invoke theorems without checking preconditions. For instance, in the series convergence problem (Fig.~\ref{fig:qualitaitve-e}), GPT-OSS-120B applied the Alternating Series Test to a sequence that does not alternate term-by-term, and the PRM accepted it. Similarly, in an inequality problem, the solver applied Jensen’s inequality outside its domain of concavity, and \shortname{} again failed to flag the mistakes. We hypothesize that $G$ struggles to generate adversarial negatives that specifically target theorem preconditions, this might be a limitation due to the size of the backbone model used.

%% file: Related_Work/reward_models.tex
\paragraph{Mathematical Reasoning with LLMs}  


LLM reasoning has advanced through chain-of-thought prompting~\citep{CoT}, self-consistency~\citep{self-consistency}, and outcome-level reinforcement learning~\citep{deepseekr1, OpenAIo1SystemCard2024}. Outcome-based reward models~\citep{orm1, orm2} give sparse feedback~\citep{prime} and can reinforce faulty reasoning. Process Reward Models (PRMs)~\citep{omegaPRM, prime, GLoRE, verifystepbystep, AutoPSV, pqm} address this by supervising intermediate steps, but require costly human annotations. Synthetic data generation methods~\citep{Math-Shepherd, zhang2024rest, guan2025rstar, sun2025freeprmtrainingprocessreward, du2025mmprmenhancingmultimodalmathematical, li2025fgprmfinegrainedhallucinationdetection} generate synthetic step-level data, yet rely on final-answer correctness, which fails when wrong steps yield right answers.   
\shortname{} instead trains a generator adversarially to produce hard, plausibly incorrect steps that target the PRM’s weaknesses. This yields a dynamic curriculum of negatives, continuously adapting to the PRM’s blind spots rather than relying on fixed heuristics.

%% file: Related_Work/adversarial_ml.tex
\paragraph{Adversarial Training and Robustness}
Adversarial training was introduced with GANs~\citep{gan}, where a generator learns a target distribution byto fooling a discriminator, and later extended to robust optimization~\citep{madry2018towards} and GAIL~\citep{ho2016generativeadversarialimitationlearning} for policy imitation.  
\shortname{} applies this idea to LLM reasoning steps, which are discrete and require a different setup. Unlike GANs or robust optimization that use zero-sum, single-objective formulations~\citep{madry2018towards, gan}, \shortname{} defines a general-sum game with non-opposing utilities for the Generator and Reward Model. This yields equilibria that are stationary points of the game dynamics (Nash equilibria) rather than simple minimax solutions. \shortname{} also connects to self-play in multi-agent RL~\citep{baker2020emergenttoolusemultiagent} and recent advances in optimization for adversarial learning~\citep{ogda1, ogda2}. Further, more recently, \citet{advrm} applied adversarial training to outcome reward models using ensemble disagreement. In contrast, \shortname{} targets process supervision and evolves the generator and PRM in tandem, compared to utilizing a static generator.
We defer a detailed discussion to Appendix~\ref{app:related_work}.

%% file: Discussion/discussion.tex
\section{Conclusion, Limitations and Future Work}

In this work, we introduced \shortname{}, an adversarial framework to improve process supervision for LLM reasoning. Unlike prior approaches that rely on outcome-based feedback, static human-annotated datasets, or heuristic search with MCTS, \shortname{} actively optimizes for hard negatives to train the PRM. This adversarial setup exposes the PRM to subtle reasoning errors making it robust and generalizable across domains.


A limitation of our method is the computational cost. While inference-time compute is unchanged, training requires additional resources due to the adversarial optimization between generator and PRM. This cost is incurred only once, but remains higher than approaches based solely on static datasets.


As the availability of high-quality training data for large models plateaus, adversarial training offers a way to generate targeted supervision, specifically focusing on model weaknesses, where data collection is expensive or impractical. This makes the approach relevant beyond mathematical reasoning, with potential applications in scientific domains, robotics, and other areas where annotation is difficult to scale.


Future work includes extending the framework beyond the two-player formulation. Multi-agent variants could introduce diverse adversarial roles that target different error modes. Another direction is applying the approach to embodied agents and robotics, where agents must learn to operate in noisy environments induced by other agent's adversarial interactions. On the optimization side, although our analysis established convergence under OGDA with regularization, further work is needed to understand large-scale training dynamics and to develop algorithms that enhance stability.

%% file: Appendix/llm_use.tex
\subsection{LLM Aid for Writing}

We use LLMs only for making the writing concise. We provide the LLM with the text and ask it not to change any information, just make it concise. We the manually verify the text and make changes to it. Few examples are given below: 
\begin{tcolorbox}[
  enhanced,
  fonttitle=\bfseries,
  colback=mycolor!10,
  colframe=mycolor,
  arc=2mm,
  boxrule=0.8pt,
  sidebyside,
  sidebyside align=top,
  bicolor,
  colbacklower=mycolor!10,
  subtitle style={top=0pt,bottom=2pt}, 
  subtitle style={fontupper=\color{black}}
]
  \tcbsubtitle{Original}
  A future work is to extend the adversarial formulation to other game-theory settings. In this paper, we focused on a two-player setup between a generator and a reward model, but multi-agent extensions could allow for diverse adversarial roles that challenge models along different dimensions of reasoning. Additionally, while our analysis showed convergence under OGDA with regularization, more work is needed to fully understand stability in large-scale adversarial training for language models. Exploring optimization techniques that reduce training instability while retaining the benefits of adversarial training is an important next step.

  \tcblower
  \tcbsubtitle{LLM + Human Polishing}
  Future work includes extending the framework beyond the two-player formulation. Multi-agent variants could introduce diverse adversarial roles that probe different aspects of reasoning. On the optimization side, while our analysis demonstrated convergence with OGDA and regularization, further study is needed to understand stability at scale and to design algorithms that improve training efficiency while preserving the benefits of adaptive curricula.

\end{tcolorbox}

Another example: 

\begin{tcolorbox}[
  enhanced,
  fonttitle=\bfseries,
  colback=mycolor!10,
  colframe=mycolor,
  arc=2mm,
  boxrule=0.8pt,
  sidebyside,
  sidebyside align=top,
  bicolor,
  colbacklower=mycolor!10,
  subtitle style={fontupper=\color{black}},  
]

  \tcbsubtitle{Original}
    To access the effectiveness of our framework, we evaluate the accuracy of LLM solvers guided by our trained PRM against the strongest existing methods (See Table \ref{tab:math_performance}). Averaged across five  math-reasoning benchmarks, under equivalent test-time compute budgets,
    \shortname{} improves solver accuracy by ... over the strongest trained PRM baseline (ReST-MCTS) and by ... over best prompting methods (Tree-of-Thought).  We observe the highest gains on OlympiadBench, which is an out-of-distribution (OOD) benchmark for the PRM. Here, \shortname{} yields a substantial ... pp improvement over the best baseline (ReST-MCTS), which strongly indicates enhanced generalization and reduced overfitting.

  \tcblower
  \tcbsubtitle{LLM + Human Polishing}
  Table~\ref{tab:math_performance}, shows the performance LLM solvers guided by our PRM against prior methods under equivalent test-time compute. Averaged over five benchmarks, \shortname{} improves solver accuracy by $+x$ pp relative to the strongest trained PRM baseline (ReST-MCTS \cite{zhang2024rest}) and by $+x$ pp relative to the best prompting method (Tree-of-Thought \cite{treeofthought}). We observe the largest gains on \textit{OlympiadBench}, which is an out-of-distribution benchmark, where \shortname{} outperforms ReST-MCTS by $+x$ pp, suggesting better generalization beyond the training distribution.

\end{tcolorbox}

%% file: Appendix/qualitaitve_examples.tex
\subsection{Evolution of Generator Perturbations}
\begin{tcolorbox}[
  enhanced, breakable,
  colback=mycolor!10, colframe=mycolor,
  title= {\textbf{Question}: A die, with the numbers 1, 2, 3, 4, 6, and 8 on its six faces, is rolled. After this roll, if an odd number appears on the top face, all odd numbers on the die are doubled. If an even number appears on the top face, all the even numbers are halved. If the given die changes in this way, what is the probability that a 2 will appear on the second roll of the die?},
  coltitle=black,
  segmentation style={dashed} 
]

\textbf{Correct Step}: The die has six faces. There are two odd numbers {1, 3} and four even numbers {2, 4, 6, 8}. Therefore, the probability of rolling an odd number is 2/6 = 1/3, and the probability of rolling an even number is 4/6 = 2/3.\\
\tcbline
\textbf{t = 500}: The die has six faces. There are \colorbox{orange!35!white}{three odd} numbers {1, 3} and \colorbox{orange!35!white}{three even} numbers {2, 4, 6, 8}. Therefore, the probability of rolling an odd number is \colorbox{orange!35!white}{3/6 = 1/2}, and the probability of rolling an even number is \colorbox{orange!35!white}{3/6 = 1/2}.\\

\textbf{Comment}: Trivial mistake of counting number of odds wrong.
\tcbline
\textbf{t =  1000}: We first determine the two possible states of the die for the second roll. If an odd number (1 or 3) is rolled first, \colorbox{orange!35!white}{all numbers on the die} are doubled, resulting in a new die with faces {2, 4, 6, 8, 12, 16}.\\

\textbf{Comment}: wrongly interpreted question, seemingly correct step on it’s own but wrong in context of the problem.
\tcbline
\textbf{t =  3000}: We must find the probability of rolling a 2. There are two scenarios. First roll is odd (P=1/3): The die becomes {2, 2, 6, 4, 6, 8}. The probability of rolling a 2 here is 2/6. First roll is even (P=2/3): The die becomes {1, 1, 3, 2, 3, 4}. The probability of rolling a 2 here is 1/6. To find the total probability, \colorbox{orange!35!white}{we average} \colorbox{orange!35!white}{these two probabilities: (2/6 + 1/6) / 2 = (3/6) / 2 = 1/4.}\\

\textbf{Comment}: This error is deceptive because every calculation is correct, there is a conceptual error, probabilities are not averaged like this. Weighted average should be used here.

\label{fig:qualitative-a}
\end{tcolorbox}

\subsection{PRM across solver model sizes}

\begin{tcolorbox}[
  enhanced, breakable,
   colback=mycolor!10, colframe=mycolor,
  title= {\textbf{Question}: A gecko is in a room that is 12 feet long, 10 feet wide and 8 feet tall. The gecko is currently on a side wall ($10'$ by $8'$), one foot from the ceiling and one foot from the back wall ($12'$ by $8'$). The gecko spots a fly on the opposite side wall, one foot from the floor and one foot from the front wall. What is the length of the shortest path the gecko can take to reach the fly assuming that it does not jump and can only walk across the ceiling and the walls? Express your answer in simplest radical form.},
  coltitle=black,
  segmentation style={dashed} 
]
\textbf{GPT-OSS-20B}\\

We need shortest path across surfaces of a rectangular box. Geometric unfolding problem. Gecko on side wall: side wall dimensions 10' by 8'. Which side? They say side wall (10' by 8') one foot from ceiling and one foot from back wall (12' by 8'). So coordinate system: room dimensions: length 12 (x), width 10 (y), height 8 (z). Let's set x along length (front to back), y along width (left to right), z vertical. \colorbox{orange!35!white}{Back wall: plane x=12? Actually back wall is 12' by 8' (length x height). So back}
\colorbox{orange!35!white}{wall is at x=12, facing front. Good. }
\tcbline
\textbf{GPT-OSS-120B}\\

Step 1 – Set up a coordinate system

Let  

\begin{itemize}
    \item the \(x\)-axis run from the front wall \((x=0)\) to the back wall \((x=12)\);
    \item the \(y\)-axis run from the left side wall \((y=0)\) to the right side wall \((y=10)\);
    \item the \(z\)-axis run from the floor \((z=0)\) to the ceiling \((z=8)\).
\end{itemize}

\colorbox{orange!35!white}{The gecko starts on the left side wall \((y=0)\) one foot below the ceiling and one }
\colorbox{orange!35!white}{foot in front of the back wall:S=(x,y,z)=(11,0,7).}
\label{fig:qualitaitve-b}
\end{tcolorbox}
\subsection{Comparision to ReST-MCTS}
\begin{tcolorbox}[
enhanced, breakable,
 colback=mycolor!10, colframe=mycolor,
title={\textbf{Question:} Evaluate
\[
F(x) = \frac{x^5 - 1}{x - 1}
\]
at $x=1$.},
coltitle=black,
segmentation style={dashed}
]

Factor the numerator:
\colorbox{orange!35!white}{$
\frac{x^5 - 1}{x - 1} 
= \frac{(x-1)(x^4+x^3+x^2+x+1)}{x-1}
= x^4+x^3+x^2+x+1.
$}
Now substitute $x=1$:
\[
F(1) = 1^4+1^3+1^2+1+1 = 5.
\]

\textbf{Comment:} The cancellation step is invalid at $x=1$, because it amounts to dividing by zero. ReST-MCTS approves the solution but \shortname{}
Flags the error.
\label{fig:qualitative-c}

\end{tcolorbox}

\subsection{Domain Transfer}
\label{fig:qualitative-d}
\begin{tcolorbox}[
  enhanced, breakable,
  colback=mycolor!10, colframe=mycolor,
  title= {\textbf{Question}: The osmotic pressures of solutions of poly(vinyl chloride) (PVC) in cyclohexanone at 
$298 \ \mathrm{K}$ are given below. The pressures are expressed in terms of the heights 
of solution (of mass density $\rho = 0.980 \ \mathrm{g \, cm^{-3}}$) in balance with 
the osmotic pressure. Determine the molar mass of the polymer.

\[
\begin{array}{c|ccccc}
c \; (\mathrm{g \, dm^{-3}}) & 1.00 & 2.00 & 4.00 & 7.00 & 9.00 \\
\hline
h \; (\mathrm{cm})            & 0.28 & 0.71 & 2.01 & 5.10 & 8.00
\end{array}
\]},
  coltitle=black,
  segmentation style={dashed} 
]

Step 1 – Relate the measured height to the osmotic pressure

The height $h$ of the liquid column balances the osmotic pressure $\Pi$:
\[
\Pi = \rho g h
\]
\[
\rho = 0.980\;\text{g cm}^{-3}=980\;\text{kg m}^{-3},\qquad 
g = 9.81\;\text{m s}^{-2}
\]
\[
\rho g = 980 \times 9.81 \approx 9.61\times 10^{3}\;\text{Pa m}^{-1}
\]
\[
\colorbox{orange!35!white}{$\displaystyle 
\Pi \approx 9.61\times10^{3}\times h \ \text{Pa}
$}
\]

\begin{center}
\begin{tabular}{@{}cc@{}}
\toprule
$h$ (cm) & $\Pi$ (Pa) \\
\midrule
0.28 & 26.9 \\
0.71 & 68.2 \\
2.01 & 192.0 \\
5.10 & 489.8 \\
8.00 & 768.3 \\
\bottomrule
\end{tabular}
\end{center}
\end{tcolorbox}

\subsection{Failure Cases}
\begin{tcolorbox}[
  enhanced, breakable,
   colback=mycolor!10, colframe=mycolor,
  title= {\textbf{Question}: Determine the convergence of
    \[
    \sum_{n=1}^{\infty} \frac{(-1)^{\lfloor \sqrt{n}\rfloor}}{\sqrt{n}}.
    \]},
  coltitle=black,
  segmentation style={dashed} 
]


Let \(a_n=\frac{(-1)^{\lfloor \sqrt{n}\rfloor}}{\sqrt{n}}\).

Since \(1/\sqrt{n}\to 0\) and the terms alternate in sign, the Alternating Series Test implies convergence.
Moreover, grouping terms by blocks where \(\lfloor \sqrt{n}\rfloor\) is constant yields partial cancellations, so
\(\sum_{n=1}^{\infty} |a_n|=\sum_{n=1}^{\infty} \frac{1}{\sqrt{n}}\) is effectively controlled by these cancellations, implying absolute convergence.\\

\textbf{Comment} The sequence does not alternate term-by-term. The convergence is conditional.
\end{tcolorbox}

\begin{tcolorbox}[
  enhanced, breakable,
   colback=mycolor!10, colframe=mycolor,
  title= {\textbf{Question}: For \(a,b,c>0\) with \(a+b+c=3\), decide whether
\[
\frac{a}{1+a^2}+\frac{b}{1+b^2}+\frac{c}{1+c^2}\le \frac{3}{2}
\]
holds.},
  coltitle=black,
  segmentation style={dashed} 
]


Let \(f(x)=\frac{x}{1+x^2}\). Since $f(x)$ is concave,
by Jensen’s inequality,
\[
\frac{a}{1+a^2}+\frac{b}{1+b^2}+\frac{c}{1+c^2}
=\sum f(a)\le 3 f\!\left(\tfrac{a+b+c}{3}\right)=3f(1)=\frac{3}{2}.
\]\\

\textbf{Comment}
$f(x)$ is concave only on the interval $0,\sqrt{3}$. Jensen’s application is not guaranteed outside this interval and the proof is invalid even though the conclusion correct.
\label{fig:qualitaitve-e}
\end{tcolorbox}

%% file: Appendix/oracle.tex
\subsection{Correctness Oracle}
\label{app:oracle}


    




Given the gold step $s_i$ and a candidate step $s'_i$ produced by the generator,
the oracle computes a set of \emph{proximity filters} (ensure $s'_i$ is close to $s_i$ but not a trivial copy),
\emph{hard semantic validators} (mathematically disprove/verify correctness), \emph{structural validators}
(check symbol/number consistency), and optional \emph{execution checks}. The oracle outputs
$y(s'_i)\in\{0,1\}$ for \textsc{Incorrect} or \textsc{Correct} 

\medskip
\noindent\textbf{Proximity filters (all must pass).}
These constrain the candidate to be a plausible variant of $s_i$ (prevents degenerate far-off edits and trivial copies).
Let $\text{len}(\cdot)$ be token length, $\cos(\cdot,\cdot)$ a sentence embedding cosine similarity, and
$\mathrm{Lev}(\cdot,\cdot)$ a normalized Levenshtein distance in $[0,1]$.

\[
\begin{aligned}
&\textbf{(P1) Length window:} &&
\mathbb{1}_{\text{len}} \;=\; \mathbb{I}\!\left[\;\tfrac{1}{2}\,\text{len}(s_i) \;\le\; \text{len}(s'_i) \;\le\; 2\,\text{len}(s_i)\;\right] \\
&\textbf{(P2) Similarity band:} &&
\mathbb{1}_{\text{sim}} \;=\; \mathbb{I}\!\left[\, \tau_{\min} \le \cos(s'_i, s_i) \le \tau_{\max} \,\right],\quad (\tau_{\min}{=}0.5,\ \tau_{\max}{=}0.9) \\
&\textbf{(P3) Non-exactness:} &&
\mathbb{1}_{\text{neq}} \;=\; \mathbb{I}\!\left[\, s'_i \ne s_i \,\right] \\
&\textbf{(P4) Minimality:} &&
\mathbb{1}_{\text{min}} \;=\; \mathbb{I}\!\left[\, \mathrm{Lev}(s'_i,s_i) \le \delta_{\max} \,\right],\quad (\delta_{\max}{=}0.35)
\end{aligned}
\]
We require $\mathbb{1}_{\text{prox}} \triangleq \mathbb{1}_{\text{len}}\cdot \mathbb{1}_{\text{sim}}\cdot \mathbb{1}_{\text{neq}}\cdot \mathbb{1}_{\text{min}} = 1$. If the proximity filters pass, we move on to check for the semantic validators. 

\noindent\textbf{Hard semantic validators (decide correctness/incorrectness).}
We compile both $s_i$ and $s'_i$ to symbolic forms when applicable (algebraic expressions, equalities, transforms).
Let $\mathsf{SymEq}(\cdot)$ denote a symbolic expression extracted from a step (when unavailable, the check abstains).
\[
\begin{aligned}
&\textbf{(H1) Symbolic equivalence:} &&
\mathbb{1}_{\text{eq}} \;=\;
\begin{cases}
0 & \text{if } \mathsf{simplify}\!\left(\mathsf{SymEq}(s'_i) - \mathsf{SymEq}(s_i)\right) = 0 \\
\bot & \text{if symbolic forms unavailable} \\
1 & \text{otherwise}
\end{cases} \\
&\textbf{(H2) Numeric consistency:} &&
\mathbb{1}_{\text{num}} \;=\;
\begin{cases}
0 & \text{if } \max_{x\in\mathcal{T}} \left|\mathsf{eval}_{x}\big(s'_i\big) - \mathsf{eval}_{x}\big(s_i\big)\right| \le \epsilon \\
1 & \text{if } \exists x\in\mathcal{T} \text{ s.t. diff }>\epsilon \\
\bot & \text{if not evaluable}
\end{cases} \\
&\textbf{(H3) Algebraic legality:} &&
\mathbb{1}_{\text{alg}} \;=\; \mathbb{I}\!\left[\,\text{illegal ops (e.g., divide-by-zero at the evaluation point)}\,\right] \\
&\textbf{(H4) Theorem preconditions:} &&
\mathbb{1}_{\text{pre}} \;=\; \mathbb{I}\!\left[\,\text{some invoked preconditions don't hold}\,\right] \\
&\textbf{(H5) Units/dimensions:} &&
\mathbb{1}_{\text{unit}} \;=\; \mathbb{I}\!\left[\,\text{dimensionally inconsistent; implicit unit swaps}\,\right]
\end{aligned}
\]

\medskip
\noindent\textbf{Structural validators (local consistency).}
Let $\mathrm{Ent}(\cdot)$ be the multiset of named entities, numbers, and operators; $\mathrm{Bind}(\cdot)$ maps symbols to their roles.
\[
\begin{aligned}
&\textbf{(S1) Entity/number alignment:} &&
\mathbb{1}_{\text{ent}} \;=\; \mathbb{I}\!\left[\,\mathrm{Jaccard}\big(\mathrm{Ent}(s'_i), \mathrm{Ent}(s_i)\big) \in [\alpha_{\min}, \alpha_{\max}]\,\right] \\
&\textbf{(S2) Variable binding consistency:} &&
\mathbb{1}_{\text{bind}} \;=\; \mathbb{I}\!\left[\,\mathrm{Bind}(s'_i)\ \text{is not consistent with}\ \mathrm{Bind}(s_i)\,\right]
\end{aligned}
\]
We use $(\alpha_{\min},\alpha_{\max})=(0.5,0.95)$ to allow small edits but prevent verbatim copies or unrelated content.

We require $\mathbb{1}_{\text{eq}} \textasciicircum \mathbb{1}_{\text{num}}\textasciicircum \mathbb{1}_{\text{alg}}\textasciicircum \mathbb{1}_{\text{pre}} \textasciicircum
\mathbb{1}_{\text{unit}} \textasciicircum \mathbb{1}_{\text{ent}}\textasciicircum \mathbb{1}_{\text{blind}}= 1$ for a step to qualify as incorrect. 


\medskip

\medskip
\noindent\textbf{Thresholds.}
We set $(\tau_{\min},\tau_{\max},\delta_{\max},\alpha_{\min},\alpha_{\max},\epsilon)=(0.5,0.9,0.35,0.5,0.95,10^{-6})$ by validation.
Results are robust in a neighborhood of these values.

\subsection{KL and Entropy Regularization Prevent Mode Collapse}
\label{app:mode-collapse}

We provide formal results showing that KL- and entropy-based regularization guarantees full-support equilibria and thus prevents mode collapse.  

\begin{lemma}[Equivalence of KL Penalty and Entropy Bonus]
Let $\pi, \pi_0$ be distributions over a finite action space $\mathcal{A}$. Then
\[
-\mathrm{KL}(\pi \,\|\, \pi_0) \;=\; H(\pi) + \sum_{a \in \mathcal{A}} \pi(a)\,\log \pi_0(a),
\]
where $H(\pi) = -\sum_{a\in\mathcal{A}} \pi(a)\log \pi(a)$ is the Shannon entropy.
\end{lemma}


\begin{lemma}[Form of the Regularized Optimum]
Consider the objective
\[
J(\pi) = \sum_{a \in \mathcal{A}} \pi(a) V(a) \;-\;\tau \,\mathrm{KL}(\pi \,\|\, \pi_0),
\]
where $V(a)$ is the expected payoff or reward of action $a$. Then the maximizer is
\[
\pi^\ast(a) = \frac{1}{Z} \,\pi_0(a)\,\exp\!\Big(\tfrac{1}{\tau} V(a)\Big),
\]
with normalization constant $Z = \sum_{a}\pi_0(a)\exp(V(a)/\tau)$.
\end{lemma}

\begin{proof}
By Lemma 1,
\[
J(\pi) = \sum_a \pi(a) V(a) + \tau H(\pi) + \tau \sum_a \pi(a)\log \pi_0(a).
\]
This is strictly concave in $\pi$. Introducing a Lagrangian with multiplier $\lambda$ for the constraint $\sum_a \pi(a)=1$:
\[
\mathcal{L}(\pi,\lambda) = \sum_a \pi(a) V(a) + \tau H(\pi) + \tau \sum_a \pi(a)\log \pi_0(a) + \lambda \Big(1 - \sum_a \pi(a)\Big).
\]
Taking derivatives:
\[
\frac{\partial \mathcal{L}}{\partial \pi(a)} = V(a) - \tau \log \pi(a) + \tau \log \pi_0(a) - \lambda.
\]
Setting to zero and solving for $\pi(a)$ yields
\[
\pi(a) \propto \pi_0(a)\exp\!\big(V(a)/\tau\big).
\]
Normalization gives the Gibbs--Boltzmann form.
\end{proof}

\begin{definition}[Support of a Probability Distribution]
Given a probability distribution $P$ over a set $\Omega$, the support of $P$, denoted $\text{supp}(P)$, is the set of all elements in $\Omega$ that are assigned a non-zero probability by $P$:
$$ \text{supp}(P) = \{ \omega \in \Omega \mid P(\omega) > 0 \}. $$
\end{definition}

Now, we combine these to define full support for a policy, specifically in the context of an equilibrium:

\begin{definition}[Full Support of an Equilibrium Policy]
An equilibrium policy $\pi_i^*$ for player $i$ (part of a Nash Equilibrium $\pi^* = (\pi_1^*, \dots, \pi_N^*)$) is said to have \textbf{full support} if every pure strategy available to player $i$ has a non-zero probability of being played under that policy. Formally, for all $s \in S_i$:
$$ \pi_i^*(s) > 0. $$
Equivalently, the support of the equilibrium policy is the entire pure strategy set: $\text{supp}(\pi_i^*) = S_i$.
\end{definition}

\noindent In simpler terms, if an equilibrium policy has full support, it means the player randomizes over all their available actions, never completely ruling out any single pure strategy.

\begin{theorem}[Full Support of the Equilibrium Policy]
If $\pi_0$ has full support over $\mathcal{A}$, then the solution $\pi^\ast$ from Lemma 2 also has full support.  
\end{theorem}

\begin{proof}
For each $a \in \mathcal{A}$,
\[
\pi^\ast(a) = \tfrac{1}{Z}\,\pi_0(a)\exp(V(a)/\tau).
\]
Since $\pi_0(a) > 0$ by assumption and the exponential is strictly positive, $\pi^\ast(a) > 0$.
\end{proof}

\begin{corollary}[Avoidance of Mode Collapse]
Under the assumption of a full-support reference policy $\pi_0$, the generator’s KL-regularized optimization problem admits a unique maximizer $\pi^\ast$ that assigns nonzero probability to every action. Hence the generator cannot collapse to a deterministic or low-support distribution.
\end{corollary}

\paragraph{Remark (Role of Temperature $\tau$).}
The parameter $\tau$ controls exploration:
\begin{itemize}
    \item As $\tau \to 0$, $\pi^\ast$ concentrates near maximizers of $V(a)$ (greedy strategy).
    \item As $\tau \to \infty$, $\pi^\ast$ approaches $\pi_0$.
\end{itemize}
Thus $\tau$ interpolates between exploitation and exploration while preserving full support.

\subsection{Compute-Performance Tradeoff}
\label{app:pareto-frontier}
The compute time (GPU hours) required to train ReST-MCTS for 3 epochs (as recommended in the paper) and \shortname{} for 5000 time steps (1 epoch). As we can see in the Figure ~\ref{fig:compute-perf-tradeoff}, for a $0.25$x increase in training time, \shortname{} gives a $4.6\%$ performance improvement. 

\begin{figure}[h!]
    \centering
    \includegraphics[width=0.5\linewidth]{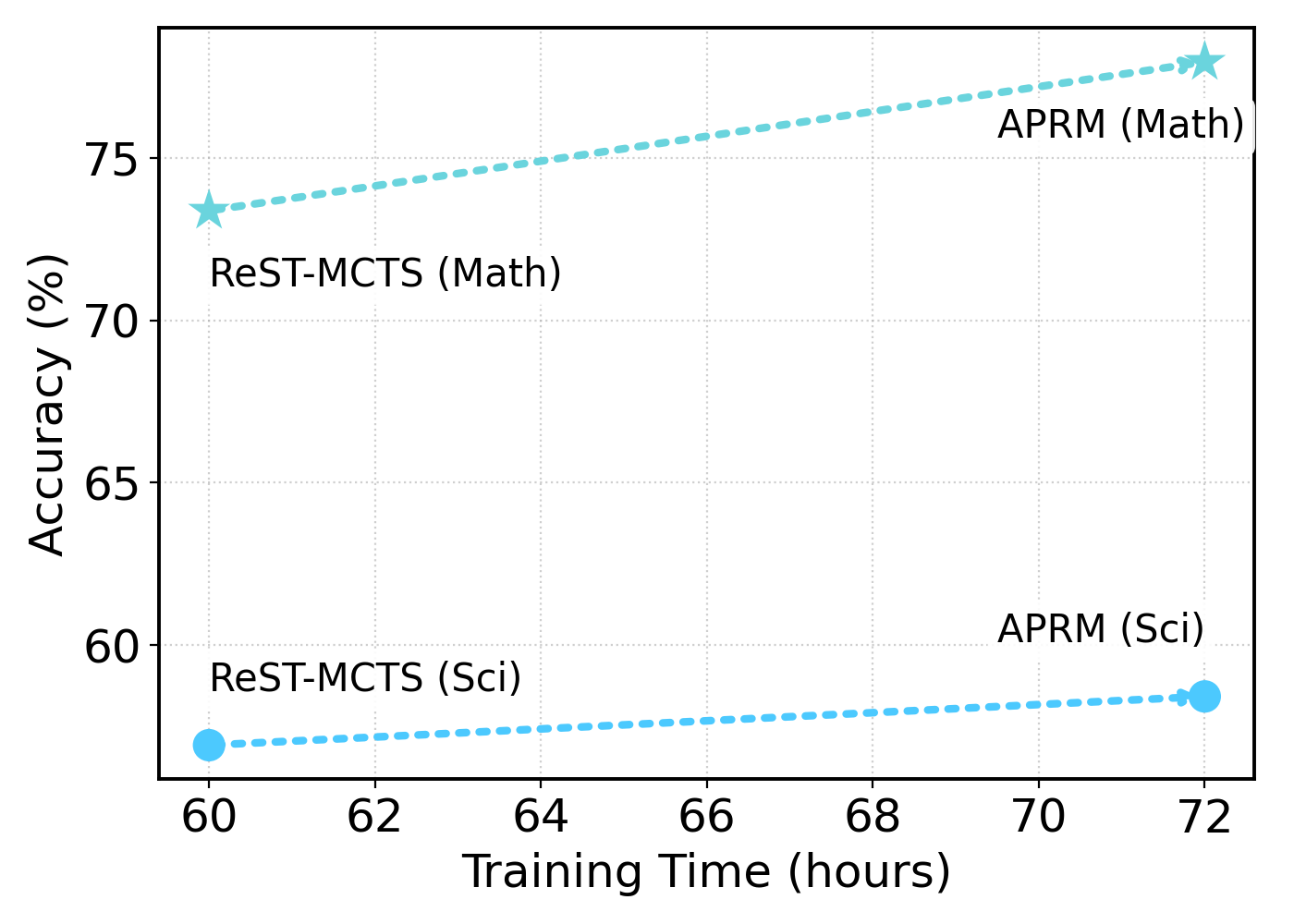}
    \caption{Training compute vs. accuracy on Math and SciBench. Dotted arrows connect ReST-MCTS (60 h) to \shortname{} (72 h), showing the direction of improvement. \shortname{} achieves higher accuracy on both benchmarks at slightly higher compute.}
    \label{fig:compute-perf-tradeoff}
\end{figure}

\subsection{Proof of Theorem 2}
\label{app:thm2}

\textbf{Proof of Linear Time Convergence for OGDA with Strongly Monotone Operator}

In our game, players' strategies are $\pi = (\pi_\theta, \pi_\phi) \in \mathcal{\Pi}$, and their coupled dynamics are governed by a game operator $F: \Pi \to \Pi$. A Nash Equilibrium $z^*=(\theta^*, \phi^*)$ is a fixed point such that $F(z^*) = 0$.

\begin{definition}[Strong Monotonicity]
An operator $F: \Pi \to \Pi$ is $\mu$-strongly monotone if there exists $\mu > 0$ such that for all $x, y \in \Pi$:
$$ \langle F(x) - F(y), x - y \rangle \ge \mu \|x - y\|^2. $$
\end{definition}

\begin{definition}[Lipschitz Continuity]
An operator $F \text( is )L$-Lipschitz continuous if there exists $L > 0$ such that for all $x, y \in \Pi$:
$$ \|F(x) - F(y)\| \le L \|x - y\|. $$
\end{definition}

Given a strongly monotone operator, there exists a unique $z^*$ such that $F(z^*) = 0$.
The Optimistic Gradient Descent Ascent (OGDA) update rule is given by:
$$ z^{(t+1)} = z^{(t)} - \eta \left( 2 F(z^{(t)}) - F(z^{(t-1)}) \right), $$
where $\eta > 0$ is the learning rate, and for $t=0$, we set $z^{(-1)} = z^{(0)}$.

\begin{theorem}[Linear Convergence of OGDA]
Let $F(z)$ be an operator that is $\mu$-strongly monotone and $L$-Lipschitz continuous. Then, for a sufficiently small learning rate $\eta > 0$ (specifically, $\eta \le \nicefrac{\mu}{L^2}$), the OGDA algorithm converges linearly to the unique Nash Equilibrium $z^*$ satisfying $F(z^*) = 0$. That is, there exist constants $C > 0$ and $\rho \in (0, 1)$ such that the distance to the equilibrium satisfies:
$$ \|z^{(t)} - z^*\|^2 \le C \rho^t. $$
\end{theorem}

\begin{proof}
Let $e^{(t)} = z^{(t)} - z^*$ denote the error at iteration $t$. Since $F(z^*) = 0$, we can rewrite the OGDA update rule in terms of the error:
$$ e^{(t+1)} = e^{(t)} - \eta \left( 2 (F(z^{(t)}) - F(z^*)) - (F(z^{(t-1)}) - F(z^*)) \right). $$
Let $G^{(t)} = F(z^{(t)}) - F(z^*)$. Then,
$$ e^{(t+1)} = e^{(t)} + \eta (2 G^{(t)} - G^{(t-1)}). $$
Consider a quadratic Lyapunov function $V_t = \|e^{(t)}\|^2 + \alpha \|e^{(t)} - e^{(t-1)}\|^2$ for some $\alpha > 0$ to be chosen. For simplicity and conciseness, we present a direct argument for contraction.

By analyzing the dynamics of the error $e^{(t)}$, one can establish a contraction property. The proof typically involves studying the evolution of a combined state vector, e.g., $[e^{(t)}; e^{(t-1)}]$, or by carefully choosing a specific quadratic potential function. Following established results for OGDA on strongly monotone operators, we know that for a sufficiently small learning rate $\eta$, the algorithm exhibits a contraction.

Specifically, for $\eta$ chosen such that $\eta^2 L^2 < \mu \eta$, and with additional conditions on $\eta$ to stabilize the "optimistic" step, it can be shown that there exist constants $\alpha_1, \alpha_2 > 0$ and $\rho \in (0,1)$ such that for a suitably defined potential function (e.g., $W_t = \|e^{(t)}\|^2 + \alpha_1 \langle e^{(t)}, e^{(t-1)} \rangle + \alpha_2 \|e^{(t-1)}\|^2$):
$$ W_{t+1} \le \rho W_t. $$
This contraction arises from leveraging both the strong monotonicity and Lipschitz continuity properties. The strong monotonicity term $\langle F(z^{(t)}) - F(z^*), z^{(t)} - z^* \rangle \ge \mu \|z^{(t)} - z^*\|^2$ provides a robust "pull" towards the equilibrium. The Lipschitz continuity bounds the change in the gradient, controlling the step size. The "optimistic" term $F(z^{(t-1)})$ helps in dampening oscillations inherent in standard gradient ascent/descent for games.

For $\eta \in (0, \nicefrac{\mu}{L^2})$, a more precise analysis (e.g., as in \cite{gidel2019variational, liang2019linear, mokhtari2020unified}) demonstrates that the sequence of iterates $\{z^{(t)}\}$ converges linearly to $z^*$. The specific choice of $\eta$ balances the strong monotonicity and Lipschitz constant. For instance, in some analyses, choosing $\eta < \nicefrac{1}{L}$ and further restricting it based on $\mu$ ensures contraction.

The linear convergence, also known as exponential convergence, implies that the error shrinks by a constant factor $\rho < 1$ at each iteration, leading to the bound $\|z^{(t)} - z^*\|^2 \le C \rho^t$.
\end{proof}

%% file: Appendix/related_work.tex
\subsection{Related Work}
\label{app:related_work}
\input{Related_Work/game_theory}

%% file: Related_Work/game_theory.tex
\paragraph{Multi-Agent RL and Self-Play}
The Multi-Agent Reinforcement Learning (MARL) and Self-Play literature offers powerful paradigms for learning complex behaviors through competitive interaction. Works like~\citep{baker2020emergenttoolusemultiagent, marl-sumo} demonstrate how co-evolving agents in self-play can induce emergent complexity and an adaptive curriculum. This principle is famously seen in agents trained for games like Go or Dota. While \shortname{} leverages self-play mechanisms to create an adaptive curriculum, it's uniqueness lies in different objectives for the two players.

\paragraph{Optimization for Adversarial Learning}

Optimization in such settings, particularly non-monotone ones, is notoriously challenging. Standard Gradient Descent-Ascent (SGD)~\citep{sgd} can exhibit unstable cycling behavior. To address this, Extra-Gradient (EG)~\citep{eg} / Mirror-Prox methods and Optimistic Mirror Descent (OMD)~\citep{omd}/ Optimistic Gradient Descent-Ascent (OGDA)~\citep{ogda1, ogda2} were developed to stabilize dynamics in variational inequalities (VI)~\citep{gidel2018variational, gidel2019variational} and saddle-point problems. 
Recent work by~\citep{azar2024general, munos2024nash} demonstrate how policy regularization can induce strong monotonicity in preference learning games, enabling direct convergence with optimistic algorithms. These works are crucial to our work as they provide theoretical guarantees for the convergence of our algorithm. 

%% file: Appendix/prm_perf.tex
\subsection{Performance on MATH}

\begin{table}
  \centering
  \caption{Performance comparison of \shortname{} with the prior methods on math reasoning benchmarks. }
  \label{tab:math-full}
\resizebox{\linewidth}{!}{
  \begin{tabular*}{\linewidth}{@{\extracolsep{\fill}}l*{6}{S}}
    \toprule
    \textbf{Method}
    & \textbf{MATH500}
    & \textbf{JEEB.}
    & \textbf{OlympiadB.}
    & \textbf{AIME25}
    & \textbf{AMC} 
    & \textbf{Avg.}\\
    \midrule

     \textbf{\texttt{GPT-OSS-120B}}\\

   \hphantom{xxxx}CoT-SC              & 84.4 & 62.8 & 83.9 & 85.5 & 60.0 & 77.2 \\
    \hphantom{xxxx}LLM-J      & 85.2 & 66.0 & 83.1 & 90.4 & 80.0 & 79.3\\
    \hphantom{xxxx}ToT          & 87.6 & 66.4 & 84.3 & 89.2 & 66.7 & 79.9 \\
    \hphantom{xxxx}AutoPSV     & 86.0 & 65.0 & 84.0 & 88.0 & 65.0 & 78.5\\
    \hphantom{xxxx}Verify SbS. & 85.0 & 63.5 & 83.5 & 86.0 & 61.0 & 77.8\\
    \hphantom{xxxx}Math S.     & 85.8 & 64.5 & 83.8 & 87.5 & 62.0 & 78.3\\
    \hphantom{xxxx}rStar Math  & 90.0 & 68.0 & 84.3 & 90.7 & 67.0 & 80.8\\
    \hphantom{xxxx}ReST-MCTS        & \textbf{91.6} & 69.1 & 84.3 & 91.0 & 67.4 & 81.7 \\ 

    \rowcolor{mycolor!20}
    \hphantom{xxxx}\textbf{\shortname{}} (Ours) & 91.4 & \textbf{70.3} & \textbf{89.4} & \textbf{94.5} & \textbf{70.7} & \textbf{83.0}\\

    \midrule

    \textbf{\texttt{GPT-OSS-20B}}\\

    \hphantom{xxxx}CoT-SC              & 78.0 & 48.4 & 61.0 & 69.9 & 33.3 & 62.2 \\
    \hphantom{xxxx}LLM-J      & 85.2 & 68.0 & 83.9 & 92.8 & 60.0 & 80.3 \\
    \hphantom{xxxx}ToT          & 87.6 & 67.6 & 85.6 & 90.4 & 66.7 & 80.7 \\
    \hphantom{xxxx}AutoPSV     & 84.0 & 60.0 & 75.0 & 85.0 & 50.0 & 70.0 \\
    \hphantom{xxxx}Verify SbS. & 80.0 & 55.0 & 70.0 & 80.0 & 45.0 & 68.0 \\
    \hphantom{xxxx}Math S.     & 82.0 & 58.0 & 73.0 & 82.0 & 48.0 & 71.5 \\
    \hphantom{xxxx}rStar Math  & 89.0 & 68.2 & 84.3 & 90.8 & 65.5 & 81.5 \\
    \hphantom{xxxx}ReST-MCTS        & \textbf{91.0} & 68.3 & 84.7 & 91.3 & 66.0 & 82.3 \\ 

    \rowcolor{mycolor!20}
    \hphantom{xxxx}\textbf{\shortname{}} (Ours) & \textbf{91.0} & \textbf{73.0} &\textbf{ 90.4} & \textbf{93.6} & \textbf{68.0} & \textbf{85.0} \\

    \midrule

    \textbf{\texttt{Gemma-3-27B}}\\
    \hphantom{xxxx}CoT-SC              & 78.8 & 48.8 & 48.7 & 69.9 & 13.3 & 58.4 \\
    \hphantom{xxxx}LLM-J      & 86.4 & 69.2 & 83.1 & 91.6 & 73.3 & 80.4 \\
    \hphantom{xxxx}ToT          & 88.0 & 67.6 & 85.2 & 88.0 & 60.0 & 80.2 \\
    \hphantom{xxxx}AutoPSV     & 83.0 & 60.0 & 70.0 & 80.0 & 40.0 & 65.0 \\
    \hphantom{xxxx}Verify SbS. & 80.0 & 55.0 & 60.0 & 75.0 & 30.0 & 63.0 \\
    \hphantom{xxxx}Math S.     & 81.5 & 58.0 & 65.0 & 78.0 & 35.0 & 66.0 \\
    \hphantom{xxxx}rStar Math  & 89.5 & 69.8 & 83.8 & 89.5 & 60.5 & 72.0 \\
    \hphantom{xxxx}ReST-MCTS        & \textbf{93.2} & 70.3 & 84.3 & 90.6 & \textbf{61.2} & 72.6 \\ 
    \rowcolor{mycolor!20}
    \hphantom{xxxx}\textbf{\shortname{}} (Ours)& 91.4 & \textbf{74.4} & \textbf{90.7} & \textbf{91.9} & 60.8 & \textbf{85.2} \\
  
    \midrule

    \textbf{\texttt{Gemma-3-12B}}\\
    \hphantom{xxxx}CoT-SC              & 76.8 & 43.2 & 35.6 & 53.0 & 23.3 & 51.4 \\
    \hphantom{xxxx}LLM-J      & 72.8 & 40.4 & 37.7 & 55.4 & 20.0 & 49.9 \\
    \hphantom{xxxx}ToT          & 77.6 & 49.2 & 40.7 & 61.4 & 20.0 & 55.3 \\
    \hphantom{xxxx}AutoPSV     & 78.0 & 45.0 & 38.0 & 58.0 & 20.1 & 54.0 \\
    \hphantom{xxxx}Verify SbS. & 77.5 & 42.0 & 36.0 & 55.0 & 20.05 & 52.0 \\
    \hphantom{xxxx}Math S.     & 78.5 & 44.0 & 37.0 & 57.0 & 20.1 & 53.0 \\
    \hphantom{xxxx}rStar Math  & 79.0 & 50.0 & 39.0 & 62.0 & 20.1 & 56.0 \\
    \hphantom{xxxx}ReST-MCTS        & \textbf{81.4} & 50.7 & 40.3 & 62.6 & 20.2 & 56.9 \\ 
    \rowcolor{mycolor!20}
    \hphantom{xxxx}\textbf{\shortname{}} (Ours) & 80.0 & \textbf{53.1} & \textbf{44.0} &\textbf{ 65.7 }& \textbf{21.6} & \textbf{58.6} \\

    \bottomrule
  \end{tabular*}} 
  \vspace{-1em}
\end{table}

\newpage
\subsection{Performance on PRM Benchmarks}
\begin{figure}[h]
    \centering
    \begin{subfigure}{0.45\textwidth}
        \centering
        \includegraphics[width=\textwidth]{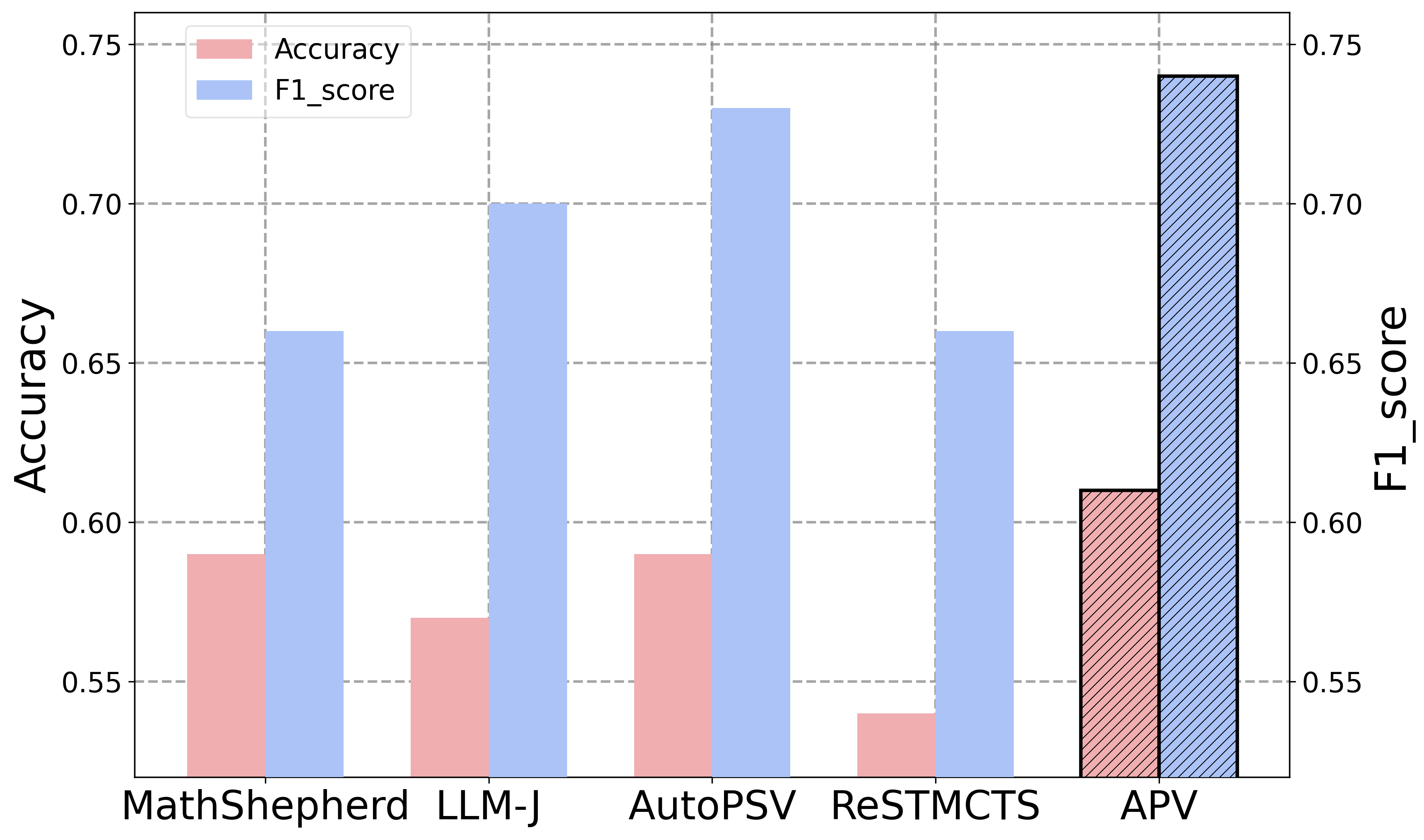}
        \caption{Accuracy}
    \end{subfigure}
    \hfill
    \begin{subfigure}{0.45\textwidth}
    \centering
     \includegraphics[width=\textwidth]{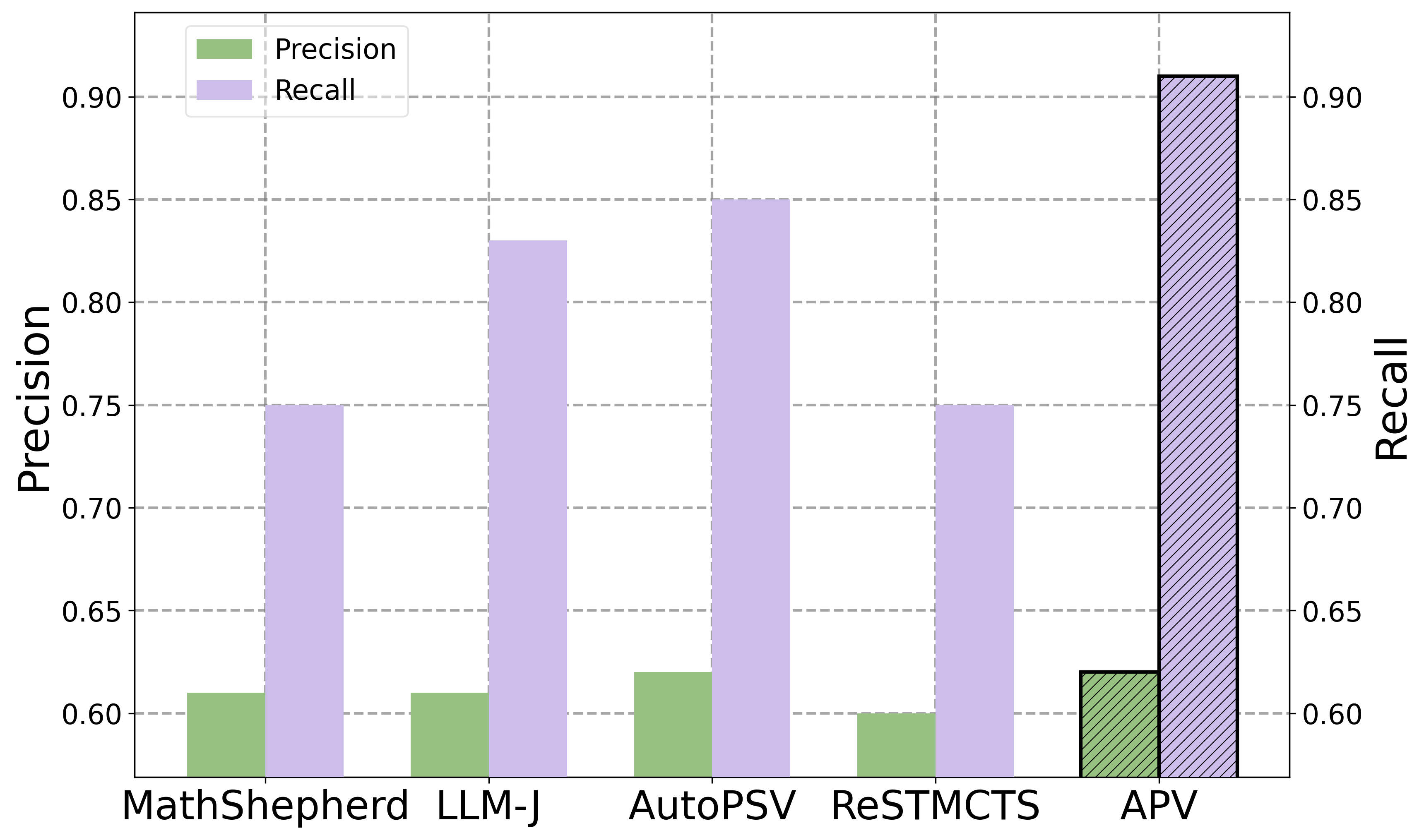}
     \caption{F1-Score}
  \end{subfigure}
    \hfill
    \caption{Comparison of performance on a subset of PRM800K \citep{prm800k}. \shortname{} (hatched) achieves the highest accuracy, precision, recall, and F1 score, indicating stronger alignment with humans}
    \label{fig:rewards_and_metrics}
\end{figure}

\subsection{Training Details}
\label{app:hyperparam}

\subsubsection{Hyperparameters to train \shortname{}}
Embedding Model (for correctness oracle): sentence-transformers/all-MiniLM-L6-v2\\
PPO Hyperparameters (Reinforcement Learning):\\
Learning Rate$\eta$: 1e-6\\
Optimizer: OGDA\\
PPO Batch Size : 4\\
PPO Epochs (mini epoch`): 2\\
$\tau$ :0.01\\
$c_H$: 0.01\\
Clipping Range ($\epsilon$ for PPO loss):0.2\\
Discount Factor ($\gamma$): 0.99\\
Generalized Advantage Estimation (GAE) Lambda ($\lambda$):0.95\\
Maximum New Tokens: 256;
Mixed Precision used bf16 managed by accelerate.

\subsubsection{Additional Implementation Details}

\paragraph{1. ARM Baseline Implementation}
\begin{itemize}
    \item \textbf{Backbone:} We used \texttt{Llama-3.1-8B-Instruct} as the backbone for both the Process Reward Models (PRMs) and the Generator.
    \item \textbf{Outcome RM Training:} We trained two separate PRMs on the PRM800K~\cite{prm800k} dataset using the Bradley-Terry objective with different random seeds. 
    \textit{Note: While the original ARM paper ~\cite{advrm} trains Outcome Reward Models, we train Process Reward Models to suit our task setup, preventing the use of the exact original training data.}
    \begin{itemize}
        \item \textbf{Optimization:} AdamW optimizer, Learning Rate $1 \times 10^{-6}$, Batch Size 64, trained for 1 Epoch.
    \end{itemize}
    \item \textbf{Generator Training:} We employed RLOO (Reinforce Leave-One-Out) to train the generator, as specified in the original paper.
    \begin{itemize}
        \item \textbf{Optimization:} AdamW optimizer, Learning Rate $5 \times 10^{-7}$, Batch Size 64. 
        \item \textbf{Sampling:} We generated 16 samples per prompt during training.
    \end{itemize}
    \item \textbf{Final PRM Training Data Selection:} To construct the training data for the final model, we generated 16 samples per prompt. We calculated the score for each sample using the current PRM version and filtered completions with a score lower than the average for that prompt. We then applied variance-based filtering, selecting samples with $z > 1.96$.
\end{itemize}

\paragraph{2. RL Post-Training (Using Trained PRMs for RL)}
We evaluated the utility of our trained PRMs by using them as reward signals for reinforcement learning.
\begin{itemize}
    \item \textbf{Model:} \texttt{Gemma-3-12B}, supervised by the trained PRMs in addition to a final correctness reward.
    \item \textbf{Algorithm:} Group Relative Policy Optimization (GRPO).
    \item \textbf{LoRA Configuration:}
    \begin{itemize}
        \item Rank $r=8$, Alpha $\alpha=16$.
        \item \textbf{Target Modules:} All linear layers.
    \end{itemize}
    \item \textbf{Hyperparameters:}
    \begin{itemize}
        \item \textbf{Learning Rate:} $5 \times 10^{-7}$ (AdamW).
        \item \textbf{Group Size:} 16.
        \item \textbf{KL Coefficient:} $\beta = 0.01$.
        \item \textbf{Clip Ratio:} $\epsilon = 0.2$.
    \end{itemize}
    \item \textbf{Reward Formulation:} Since GRPO requires the starting point to be consistent within the group, we calculated the reward as the sum of the PRM probabilities across all reasoning steps.
\end{itemize}

%% file: iclr2026_conference.bib
@misc{pqm,
      title={Process Reward Model with Q-Value Rankings}, 
      author={Wendi Li and Yixuan Li},
      year={2025},
      eprint={2410.11287},
      archivePrefix={arXiv},
      primaryClass={cs.CL},
      url={https://arxiv.org/abs/2410.11287}, 
}

@misc{GLoRE,
      title={GLoRe: When, Where, and How to Improve LLM Reasoning via Global and Local Refinements}, 
      author={Alex Havrilla and Sharath Raparthy and Christoforus Nalmpantis and Jane Dwivedi-Yu and Maksym Zhuravinskyi and Eric Hambro and Roberta Raileanu},
      year={2024},
      eprint={2402.10963},
      archivePrefix={arXiv},
      primaryClass={cs.CL},
      journal={arXiv preprint arXiv:2402.10963},
      url={https://arxiv.org/abs/2402.10963}, 
}

@inproceedings{
advrm,
title={Adversarial Training of Reward Models},
author={Alexander Bukharin and Haifeng Qian and Shengyang Sun and Adithya Renduchintala and Soumye Singhal and Zhilin Wang and Oleksii Kuchaiev and Olivier Delalleau and Tuo Zhao},
booktitle={Second Conference on Language Modeling},
year={2025},
url={https://openreview.net/forum?id=H6Ae8Po6fS}
}

@inproceedings{AutoPSV,
  title={AutoPSV: Automated Process-Supervised Verifier},
  author={Lu, Jianqiao and Dou, Zhiyang and Wang, Hongru and Cao, Zeyu and Dai, Jianbo and Wan, Yingjia and Guo, Zhijiang},
  booktitle={Advances in Neural Information Processing Systems},
  year={2024}
}

@article{gidel2018variational,
  title={Variational inequalities for generative adversarial networks},
  author={Gidel, Gauthier and Huang, Tony and Gabriel, Jon},
  journal={arXiv preprint arXiv:1802.04655},
  year={2018}
}

@misc{ho2016generativeadversarialimitationlearning,
      title={Generative Adversarial Imitation Learning}, 
      author={Jonathan Ho and Stefano Ermon},
      year={2016},
      eprint={1606.03476},
      archivePrefix={arXiv},
      primaryClass={cs.LG},
      url={https://arxiv.org/abs/1606.03476}, 
}

@inproceedings{scibench,
    author = {Wang, Xiaoxuan and Hu, Ziniu and Lu, Pan and Zhu, Yanqiao and Zhang, Jieyu and Subramaniam, Satyen and Loomba, Arjun R. and Zhang, Shichang and Sun, Yizhou and Wang, Wei},
    title = {{SciBench: Evaluating College-Level Scientific Problem-Solving Abilities of Large Language Models}},
    booktitle = {Proceedings of the Forty-First International Conference on Machine Learning},
    year = {2024},
    }

@misc{treeofthought,
      title={Tree of Thoughts: Deliberate Problem Solving with Large Language Models}, 
      author={Shunyu Yao and Dian Yu and Jeffrey Zhao and Izhak Shafran and Thomas L. Griffiths and Yuan Cao and Karthik Narasimhan},
      year={2023},
      eprint={2305.10601},
      archivePrefix={arXiv},
      primaryClass={cs.CL},
      url={https://arxiv.org/abs/2305.10601}, 
}

@misc{yu2024ovm,
      title={OVM, Outcome-supervised Value Models for Planning in Mathematical Reasoning}, 
      author={Fei Yu and Anningzhe Gao and Benyou Wang},
      year={2024},
      eprint={2311.09724},
      archivePrefix={arXiv},
      primaryClass={cs.AI},
      url={https://arxiv.org/abs/2311.09724}, 
}

@misc{cobbe2021training,
      title={Training Verifiers to Solve Math Word Problems}, 
      author={Karl Cobbe and Vineet Kosaraju and Mohammad Bavarian and Mark Chen and Heewoo Jun and Lukasz Kaiser and Matthias Plappert and Jerry Tworek and Jacob Hilton and Reiichiro Nakano and Christopher Hesse and John Schulman},
      year={2021},
      eprint={2110.14168},
      archivePrefix={arXiv},
      primaryClass={cs.LG},
      url={https://arxiv.org/abs/2110.14168}, 
}

@article{fei2021provably,
  title={Provably training multi-agent reinforcement learning with non-monotonic policy gradients},
  author={Fei, Mengyan and Zhang, Yuandong and Wang, Jialun},
  journal={Advances in Neural Information Processing Systems},
  volume={34},
  pages={25016--25029},
  year={2021}
}

@article{daskalakis2018gans,
  title={Gans May Have No Nash Equilibria Even With An Infinite Number Of Samples},
  author={Daskalakis, Constantinos and Panageas, Georgia},
  journal={Advances in Neural Information Processing Systems},
  volume={31},
  year={2018}
}

@inproceedings{ogda1,
  title={Optimistic gradient descent ascent for zero-sum games and connections to implicit learning},
  author={Mertikopoulos, Panayotis and Panageas, Ioannis and Zadik, Ioannis and Zhou, Gui-Hong},
  booktitle={Advances in Neural Information Processing Systems},
  volume={32},
  year={2019}
}

@inproceedings{ogda2,
  title={Training GANs is as easy as solving a minmax game},
  author={Daskalakis, Constantinos and Ilyas, Andrew and Panageas, Ioannis},
  booktitle={International Conference on Learning Representations},
  year={2018}
}

@inproceedings{Math-Shepherd,
    title = "Math-Shepherd: Verify and Reinforce {LLM}s Step-by-step without Human Annotations",
    author = "Wang, Peiyi  and
      Li, Lei  and
      Shao, Zhihong  and
      Xu, Runxin  and
      Dai, Damai  and
      Li, Yifei  and
      Chen, Deli  and
      Wu, Yu  and
      Sui, Zhifang",
    editor = "Ku, Lun-Wei  and
      Martins, Andre  and
      Srikumar, Vivek",
    booktitle = "Proceedings of the 62nd Annual Meeting of the Association for Computational Linguistics (Volume 1: Long Papers)",
    month = aug,
    year = "2024",
    address = "Bangkok, Thailand",
    publisher = "Association for Computational Linguistics",
    url = "https://aclanthology.org/2024.acl-long.510/",
    doi = "10.18653/v1/2024.acl-long.510",
    pages = "9426--9439",
    }

@misc{marl-sumo,
      title={Emergent Complexity via Multi-Agent Competition}, 
      author={Trapit Bansal and Jakub Pachocki and Szymon Sidor and Ilya Sutskever and Igor Mordatch},
      year={2018},
      eprint={1710.03748},
      archivePrefix={arXiv},
      primaryClass={cs.AI},
      url={https://arxiv.org/abs/1710.03748}, 
}

@article{prm800k,
      title={Let's Verify Step by Step}, 
      author={Lightman, Hunter and Kosaraju, Vineet and Burda, Yura and Edwards, Harri and Baker, Bowen and Lee, Teddy and Leike, Jan and Schulman, John and Sutskever, Ilya and Cobbe, Karl},
      journal={arXiv preprint arXiv:2305.20050},
      year={2023}
}

@misc{omegaPRM,
      title={Improve Mathematical Reasoning in Language Models by Automated Process Supervision}, 
      author={Liangchen Luo and Yinxiao Liu and Rosanne Liu and Samrat Phatale and Meiqi Guo and Harsh Lara and Yunxuan Li and Lei Shu and Yun Zhu and Lei Meng and Jiao Sun and Abhinav Rastogi},
      year={2024},
      eprint={2406.06592},
      archivePrefix={arXiv},
      primaryClass={cs.CL},
      url={https://arxiv.org/abs/2406.06592}, 
}

@article{CoT,
  title={Chain-of-Thought Prompting Elicits Reasoning in Large Language Models},
  author={Wei, Jason and Wang, Xuezhi and Schuurmans, Dale and Bosma, Maarten and Ichter, Brian and Xia, Fei and Chi, Ed and Le, Quoc and Zhou, Denny},
  journal={arXiv preprint arXiv:2201.11903},
  year={2022}
}

@misc{self-consistency,
      title={Self-Consistency Improves Chain of Thought Reasoning in Language Models}, 
      author={Xuezhi Wang and Jason Wei and Dale Schuurmans and Quoc Le and Ed Chi and Sharan Narang and Aakanksha Chowdhery and Denny Zhou},
      year={2023},
      eprint={2203.11171},
      archivePrefix={arXiv},
      primaryClass={cs.CL},
      url={https://arxiv.org/abs/2203.11171}, 
}

@misc{guan2025rstar,
    title={rStar-Math: Small LLMs Can Master Math Reasoning with Self-Evolved Deep Thinking},
    author={Xinyu Guan and Li Lyna Zhang and Yifei Liu and Ning Shang and Youran Sun and Yi Zhu and Fan Yang and Mao Yang},
    year={2025},
    eprint={2501.04519},
    archivePrefix={arXiv},
    primaryClass={cs.CL}
}

@misc{llmaj,
      title={Judging LLM-as-a-Judge with MT-Bench and Chatbot Arena}, 
      author={Lianmin Zheng and Wei-Lin Chiang and Ying Sheng and Siyuan Zhuang and Zhanghao Wu and Yonghao Zhuang and Zi Lin and Zhuohan Li and Dacheng Li and Eric P. Xing and Hao Zhang and Joseph E. Gonzalez and Ion Stoica},
      year={2023},
      eprint={2306.05685},
      archivePrefix={arXiv},
      primaryClass={cs.CL},
      url={https://arxiv.org/abs/2306.05685}, 
}

@misc{aime,
      title={AIME: AI System Optimization via Multiple LLM Evaluators}, 
      author={Bhrij Patel and Souradip Chakraborty and Wesley A. Suttle and Mengdi Wang and Amrit Singh Bedi and Dinesh Manocha},
      year={2024},
      eprint={2410.03131},
      archivePrefix={arXiv},
      primaryClass={cs.AI},
      url={https://arxiv.org/abs/2410.03131}, 
}

@misc{MATH,
      title={Measuring Mathematical Problem Solving With the MATH Dataset}, 
      author={Dan Hendrycks and Collin Burns and Saurav Kadavath and Akul Arora and Steven Basart and Eric Tang and Dawn Song and Jacob Steinhardt},
      year={2021},
      eprint={2103.03874},
      archivePrefix={arXiv},
      primaryClass={cs.LG},
      url={https://arxiv.org/abs/2103.03874}, 
}

@article{verifystepbystep,
      title={Let's Verify Step by Step}, 
      author={Lightman, Hunter and Kosaraju, Vineet and Burda, Yura and Edwards, Harri and Baker, Bowen and Lee, Teddy and Leike, Jan and Schulman, John and Sutskever, Ilya and Cobbe, Karl},
      journal={arXiv preprint arXiv:2305.20050},
      year={2023}
}

@misc{jeebench,
      title={Have LLMs Advanced Enough? A Challenging Problem Solving Benchmark For Large Language Models}, 
      author={Daman Arora and Himanshu Gaurav Singh and Mausam},
      year={2023},
      eprint={2305.15074},
      archivePrefix={arXiv},
      primaryClass={cs.CL},
      url={https://arxiv.org/abs/2305.15074}, 
}

@misc{OlympiadBench,
      title={OlympiadBench: A Challenging Benchmark for Promoting AGI with Olympiad-Level Bilingual Multimodal Scientific Problems}, 
      author={Chaoqun He and Renjie Luo and Yuzhuo Bai and Shengding Hu and Zhen Leng Thai and Junhao Shen and Jinyi Hu and Xu Han and Yujie Huang and Yuxiang Zhang and Jie Liu and Lei Qi and Zhiyuan Liu and Maosong Sun},
      year={2024},
      eprint={2402.14008},
      archivePrefix={arXiv},
      primaryClass={cs.CL},
      url={https://arxiv.org/abs/2402.14008}, 
}

@misc{ppo,
      title={Proximal Policy Optimization Algorithms}, 
      author={John Schulman and Filip Wolski and Prafulla Dhariwal and Alec Radford and Oleg Klimov},
      year={2017},
      eprint={1707.06347},
      archivePrefix={arXiv},
      primaryClass={cs.LG},
      url={https://arxiv.org/abs/1707.06347}, 
}

@misc{ORM1,
      title={Training Verifiers to Solve Math Word Problems}, 
      author={Karl Cobbe and Vineet Kosaraju and Mohammad Bavarian and Mark Chen and Heewoo Jun and Lukasz Kaiser and Matthias Plappert and Jerry Tworek and Jacob Hilton and Reiichiro Nakano and Christopher Hesse and John Schulman},
      year={2021},
      eprint={2110.14168},
      archivePrefix={arXiv},
      primaryClass={cs.LG},
      url={https://arxiv.org/abs/2110.14168}, 
}

@misc{ORM2,
      title={OVM, Outcome-supervised Value Models for Planning in Mathematical Reasoning}, 
      author={Fei Yu and Anningzhe Gao and Benyou Wang},
      year={2024},
      eprint={2311.09724},
      archivePrefix={arXiv},
      primaryClass={cs.AI},
      url={https://arxiv.org/abs/2311.09724}, 
}

@misc{deepseekr1,
      title={DeepSeek-R1: Incentivizing Reasoning Capability in LLMs via Reinforcement Learning}, 
      author={DeepSeek-AI},
      year={2025},
      eprint={2501.12948},
      archivePrefix={arXiv},
      primaryClass={cs.CL},
      url={https://arxiv.org/abs/2501.12948}, 
}

@misc{aimo_aimo-validation-amc_2024,
    author = {{AI-MO}},
    title = {{AI-MO/aimo-validation-amc} Dataset},
    howpublished = {\url{https://huggingface.co/datasets/AI-MO/aimo-validation-amc}},
    year = {2024},
    note = {Accessed: <Date you accessed the dataset>}
}

@article{OpenAIo1SystemCard2024,
    author = {{OpenAI}},
    title = {{OpenAI o1 System Card}},
    year = {2024},
    howpublished = {\url{https://cdn.openai.com/o1-system-card-20241205.pdf}},
    note = {Accessed: <Date you accessed the document>}
}

@article{prime,
  title={Process reinforcement through implicit rewards},
  author={Cui, Ganqu and Yuan, Lifan and Wang, Zefan and Wang, Hanbin and Li, Wendi and He, Bingxiang and Fan, Yuchen and Yu, Tianyu and Xu, Qixin and Chen, Weize and others},
  journal={arXiv preprint arXiv:2502.01456},
  year={2025}
}

@misc{gan,
      title={Generative Adversarial Networks}, 
      author={Ian J. Goodfellow and Jean Pouget-Abadie and Mehdi Mirza and Bing Xu and David Warde-Farley and Sherjil Ozair and Aaron Courville and Yoshua Bengio},
      year={2014},
      eprint={1406.2661},
      archivePrefix={arXiv},
      primaryClass={stat.ML},
      url={https://arxiv.org/abs/1406.2661}, 
}

@misc{baker2020emergenttoolusemultiagent,
      title={Emergent Tool Use From Multi-Agent Autocurricula}, 
      author={Bowen Baker and Ingmar Kanitscheider and Todor Markov and Yi Wu and Glenn Powell and Bob McGrew and Igor Mordatch},
      year={2020},
      eprint={1909.07528},
      archivePrefix={arXiv},
      primaryClass={cs.LG},
      url={https://arxiv.org/abs/1909.07528}, 
}

@article{alphageometry2024,
  title={{AlphaGeometry}: An {AI} system that finds a new way to solve {IMO} geometry problems},
  author={Lu, Xiaoxuan and Cao, Jiaming and Li, Guiding and Luan, Jianqiao and Shen, Yining and Wang, Zhichao and Xu, Jianhao and Chen, Xinyun and Cheng, Si and Deng, Jianmin and Fan, Yang and Gou, Yu and Hou, Yong and Hu, Ziniu and Jin, Hang and Li, Shuyuan and Li, Yi and Li, Ziyue and Li, Zhilin and Lu, Kai and Luo, Liangchen and Ma, Qianli and Mei, Fan and Pan, Wei and Song, Xin and Song, Yi and Sun, Youbin and Tan, Minghao and Wang, Mengdi and Wang, Siyuan and Wu, Shengkai and Xie, Tao and Xu, Runxin and Xu, Shuhao and Yang, Yi and Yang, Yun and Yu, Qian and Yuan, Lifan and Zhang, Sining and Zhang, Yuqi and Zhao, Rui and Zhou, Xiaogang and Zhu, Chenghao and Zhu, Yi and Chen, Zimin and Guo, Yimin and Wang, Yibo and Yu, Chao and Zhang, Jianwen and Zhou, Jingren and Li, Hang and Zhou, Yu},
  journal={Nature},
  volume={625},
  number={7995},
  pages={535--540},
  year={2024},
  doi={10.1038/s41586-023-06924-6}
}

@article{madry2018towards,
  title={Towards deep learning models resistant to adversarial attacks},
  author={Madry, Aleksander and Makelov, Aleksandar and Schmidt, Ludwig and Tsipras, Dimitris and Vladu, Alexandru},
  journal={arXiv preprint arXiv:1706.06083},
  year={2017} 
}

@article{azar2024general,
  title={A general theoretical paradigm to understand learning from human preferences},
  author={Azar, Mohammad Gheshlaghi and Rowland, Mark and Piot, Bilal and Guo, Daniel and Calandriello, Daniele and Valko, Michal and Munos, R{\'e}mi},
  journal={AISTATS},
  year={2024}
}

@article{munos2024nash,
  title={Nash learning from human feedback},
  author={Munos, R{\'e}mi and Jiang, Xu and Li, Guiding and Potharaju, Shaan and Sun, Yichen and Wang, Yu and Wu, Chong and Xu, Changxiao and Zeng, Qizhe and Zhang, Junyang and Zhang, Zihang},
  journal={ICML},
  year={2024}
}

@article{glicksberg1952further,
  title={A further generalization of the {Nash} equilibrium point theorem},
  author={Glicksberg, Irving L},
  journal={Pacific Journal of Mathematics},
  volume={2},
  number={2},
  pages={415--427},
  year={1952}
}

@article{chu2023investigating,
  title={Investigating the faithfulness of attention in deep models},
  author={Chu, Mengjiao and Dai, Li and Dong, Xinyi and Han, Hongxu and Song, Kai and Wen, Huaimin and Yuan, Zheng and Zhang, Bin and Chen, Yichun and Wang, Zhicheng},
  journal={arXiv preprint arXiv:2308.14081},
  year={2023}
}

@article{bommasani2021opportunities,
  title={Opportunities and Risks of Foundational Models},
  author={Bommasani, Rishi and Hudson, Drew and Adeli, Ehsan and Altman, Russ and Arora, Simran and von Arx, Sydney and Bernstein, Michael and Bohg, Jeannette and Bosselut, Antoine and Brunskill, Emma and others},
  journal={arXiv preprint arXiv:2108.07258},
  year={2021}
}

@article{lightman2023lets,
  title={Let's verify step by step},
  author={Lightman, Hunter and Kosaraju, Vineet and Burda, Yura and Edwards, Harri and Baker, Bowen and Lee, Teddy and Leike, Jan and Schulman, John and Sutskever, Ilya and Cobbe, Karl},
  journal={arXiv preprint arXiv:2305.20050},
  year={2023}
}

@article{wang2024math,
  title={Math-shepherd: Verify and reinforce {LLMs} step-by-step without human annotations},
  author={Wang, Peiyi and Li, Lei and Shao, Zhihong and Xu, Runxin and Dai, Damai and Li, Yifei and Chen, Deli and Wu, Yu and Sui, Zhifang},
  journal={Proceedings of the 62nd Annual Meeting of the Association for Computational Linguistics (Volume 1: Long Papers)},
  pages={9426--9439},
  year={2024}
}

@article{zhang2024rest,
  title={{ReST-MCTS}*: {LLM} self-training via process reward guided tree search},
  author={Zhang, Dan and Zhoubian, Sining and Hu, Ziniu and Yue, Yisong and Dong, Yuxiao and Tang, Jie},
  journal={arXiv preprint arXiv:2406.03816},
  year={2024}
}

@article{lu2024autopsv,
  title={{AutoPSV}: Automated process-supervised verifier},
  author={Lu, Jianqiao and Dou, Zhiyang and Wang, Hongru and Cao, Zeyu and Dai, Jianbo and Wan, Yingjia and Guo, Zhijiang},
  journal={Advances in Neural Information Processing Systems},
  year={2024}
}

@article{uesato2022solving,
  title={Solving math word problems with process- and outcome-based feedback},
  author={Uesato, Jonathan and Kushman, Nate and Kumar, Ramana and Song, Francis and Siegel, Noah and Wang, Lisa and Creswell, Antonia and Irving, Geoffrey and Higgins, Irina},
  journal={arXiv preprint arXiv:2204.09014},
  year={2022}
}

@misc{DeepMind2025GeminiIMO,
  author       = {Luong, Thang and Lockhart, Edward and DeepMind Team},
  title        = {Advanced Version of Gemini with Deep Think Officially Achieves Gold-Medal Standard at the International Mathematical Olympiad},
  howpublished = {Google DeepMind Blog},
  year         = {2025},
  month        = {July},
  day          = {21},
  url          = {https://deepmind.google/discover/blog/advanced-version-of-gemini-with-deep-think-officially-achieves-gold-medal-standard-at-the-international-mathematical-olympiad/},
  note         = {Accessed: YYYY-MM-DD}
}

@article{Nash1951,
  author  = {Nash, John},
  title   = {Non-cooperative Games},
  journal = {Annals of Mathematics},
  volume  = {54},
  number  = {2},
  pages   = {286-295},
  year    = {1951}
}

@article{gidel2019variational,
  title={Variational inequality approach to studying convergence of learning dynamics},
  author={Gidel, Gauthier and Laurent, Cl{\'e}ment and Bach, Francis},
  journal={arXiv preprint arXiv:1904.09340},
  year={2019}
}

@article{liang2019linear,
  title={Linear convergence of optimistic gradient descent ascent for bilinear games},
  author={Liang, Zhichao and Lai, Tze Leung and Song, Rui},
  journal={arXiv preprint arXiv:1905.10547},
  year={2019}
}

@inproceedings{mokhtari2020unified,
  title={A unified analysis of extra-gradient and optimistic gradient methods for saddle point problems},
  author={Mokhtari, Aryan and Ozdaglar, Asuman and Srebro, Nathan},
  booktitle={International Conference on Artificial Intelligence and Statistics},
  pages={3704--3714},
  year={2020},
  organization={PMLR}
}

@article{stationary1,
  title     = {Game Theoretic Optimization via Gradient-based Nikaido-Isoda Function},
  author    = {Raghunathan, Arvind U. and Cherian, Anoop and Jha, Devesh K.},
  journal   = {arXiv preprint arXiv:1905.05927},
  year      = {2019},
  url       = {https://arxiv.org/abs/1905.05927}
}

@article{stationary2,
  title     = {Policy Optimization Provably Converges to Nash Equilibria in Zero-Sum Linear Quadratic Games},
  author    = {Zhang, Kaiqing and Yang, Zhuoran and Başar, Tamer},
  journal   = {arXiv preprint arXiv:1906.00729},
  year      = {2019},
  url       = {https://arxiv.org/abs/1906.00729}
}

@misc{gpt4,
    author = {{OpenAI}},
    title = {{GPT-4 Technical Report}},
    howpublished = {\url{https://cdn.openai.com/papers/gpt-4.pdf}},
    year = {2023},
    note = {The authors list the organization OpenAI as the sole author},
    eprint = {2303.08774},
    eprinttype = {arxiv},
}

@article{sgd,
  title={A stochastic approximation method},
  author={Robbins, Herbert and Monro, Sutton},
  journal={The Annals of Mathematical Statistics},
  volume={22},
  number={3},
  pages={400--407},
  year={1951}
}

@article{eg,
  title={Extragradient method for finding saddle points and for solving other problems},
  author={Korpelevich, G M},
  journal={Ekonomika i Matematicheskie Metody},
  volume={13},
  number={4},
  pages={627--641},
  year={1977}
}

@article{omd,
  title={Optimistic mirror descent in saddle-point problems},
  author={Daskalakis, Constantinos and Daskalakis, George and Panageas, Ioannis},
  journal={arXiv preprint arXiv:1807.02629},
  year={2018}
}

@misc{sppo,
      title={Self-Play Preference Optimization for Language Model Alignment}, 
      author={Yue Wu and Zhiqing Sun and Huizhuo Yuan and Kaixuan Ji and Yiming Yang and Quanquan Gu},
      year={2024},
      eprint={2405.00675},
      archivePrefix={arXiv},
      primaryClass={cs.LG},
      url={https://arxiv.org/abs/2405.00675}, 
}

@misc{du2025mmprmenhancingmultimodalmathematical,
      title={MM-PRM: Enhancing Multimodal Mathematical Reasoning with Scalable Step-Level Supervision}, 
      author={Lingxiao Du and Fanqing Meng and Zongkai Liu and Zhixiang Zhou and Ping Luo and Qiaosheng Zhang and Wenqi Shao},
      year={2025},
      eprint={2505.13427},
      archivePrefix={arXiv},
      primaryClass={cs.AI},
      url={https://arxiv.org/abs/2505.13427}, 
}

@misc{li2025fgprmfinegrainedhallucinationdetection,
      title={FG-PRM: Fine-grained Hallucination Detection and Mitigation in Language Model Mathematical Reasoning}, 
      author={Ruosen Li and Ziming Luo and Xinya Du},
      year={2025},
      eprint={2410.06304},
      archivePrefix={arXiv},
      primaryClass={cs.CL},
      url={https://arxiv.org/abs/2410.06304}, 
}

@misc{sun2025freeprmtrainingprocessreward,
      title={FreePRM: Training Process Reward Models Without Ground Truth Process Labels}, 
      author={Lin Sun and Chuang Liu and Xiaofeng Ma and Tao Yang and Weijia Lu and Ning Wu},
      year={2025},
      eprint={2506.03570},
      archivePrefix={arXiv},
      primaryClass={cs.CL},
      url={https://arxiv.org/abs/2506.03570}, 
}
